\documentclass[onecolumn]{fairmeta}

\usepackage{amssymb}
\usepackage{mathtools}
\usepackage{amsthm}
\usepackage{amsfonts}

\usepackage[utf8]{inputenc}
\usepackage{url}
\usepackage{nicefrac}

\usepackage{adjustbox}

\usepackage{algorithm}
\usepackage{algorithmic}

\usepackage{enumitem}

\usepackage{wrapfig}

\theoremstyle{plain}
\newtheorem{theorem}{Theorem}

\theoremstyle{definition}

\theoremstyle{remark}

\newtheorem{example}{Example}

\usepackage[textsize=tiny]{todonotes}

\usepackage{tikz}

\usepackage{tikz}

\title{K-Forcing: Joint Next-K-Token Decoding via Push-Forward Language Modeling}

\author[1,2,3,*]{Zhiwei Tang}
\author[3,*]{Yuanyu He}
\author[1]{Yizheng Han}
\author[4]{Wangbo Zhao}
\author[1,2]{Jiasheng Tang}
\author[1]{Fan Wang}
\author[1,3]{Bohan Zhuang}

\affiliation[1]{DAMO Academy, Alibaba Group}
\affiliation[2]{Hupan Lab}
\affiliation[3]{Zhejiang University}
\affiliation[4]{The Hong Kong University of Science and Technology}

\contribution[*]{Equal contribution}
\metadata[Email]{\href{mailto:mcstzw@gmail.com}{mcstzw@gmail.com}, \href{mailto:bohan.zhuang@gmail.com}{bohan.zhuang@gmail.com}}
\metadata[Code]{\url{https://github.com/alibaba-damo-academy/K-Forcing}}

\abstract{
Autoregressive (AR) language modeling is the dominant paradigm for text
generation, yet its sequential token-by-token decoding makes inference
memory-bound and inefficient. Existing acceleration approaches, such as
speculative decoding and diffusion language models, can yield speedups under
certain conditions but do not directly address high-load \emph{batch
serving}---the scenario most critical for industrial-scale deployment.
We introduce \textbf{K-Forcing}, a push-forward language modeling paradigm for
\emph{joint next-$k$-token decoding}. K-Forcing distills an existing AR model
into a conditional push-forward mapping---one that transforms independent
uniform noise variables into a joint sample of multiple future tokens in a
single forward pass. This design preserves fixed-length outputs, reuses the AR
teacher backbone, and remains compatible with standard AR serving
infrastructure. We train this mapping via \emph{progressive self-forcing distillation},
which gradually expands the prediction window while enabling the student to
closely match the sequence distribution of the AR teacher.
We evaluate K-Forcing on LM1B and OpenWebText using a standard causal
Transformer backbone. When aggressively configured to generate $k=4$ tokens per forward
pass, K-Forcing delivers approximately $2.4$--$3.5\times$ speedup across
different batch sizes, while incurring modest quality degradation relative
to its AR teacher.  As inference increasingly dominates the lifetime compute cost
of modern LLMs, K-Forcing offers a promising route toward accelerating AR
generation under real-world high-load deployment.
}

\begin{document}
\maketitle

\begin{figure}[t]
  \centering
  \includegraphics[width=\textwidth]{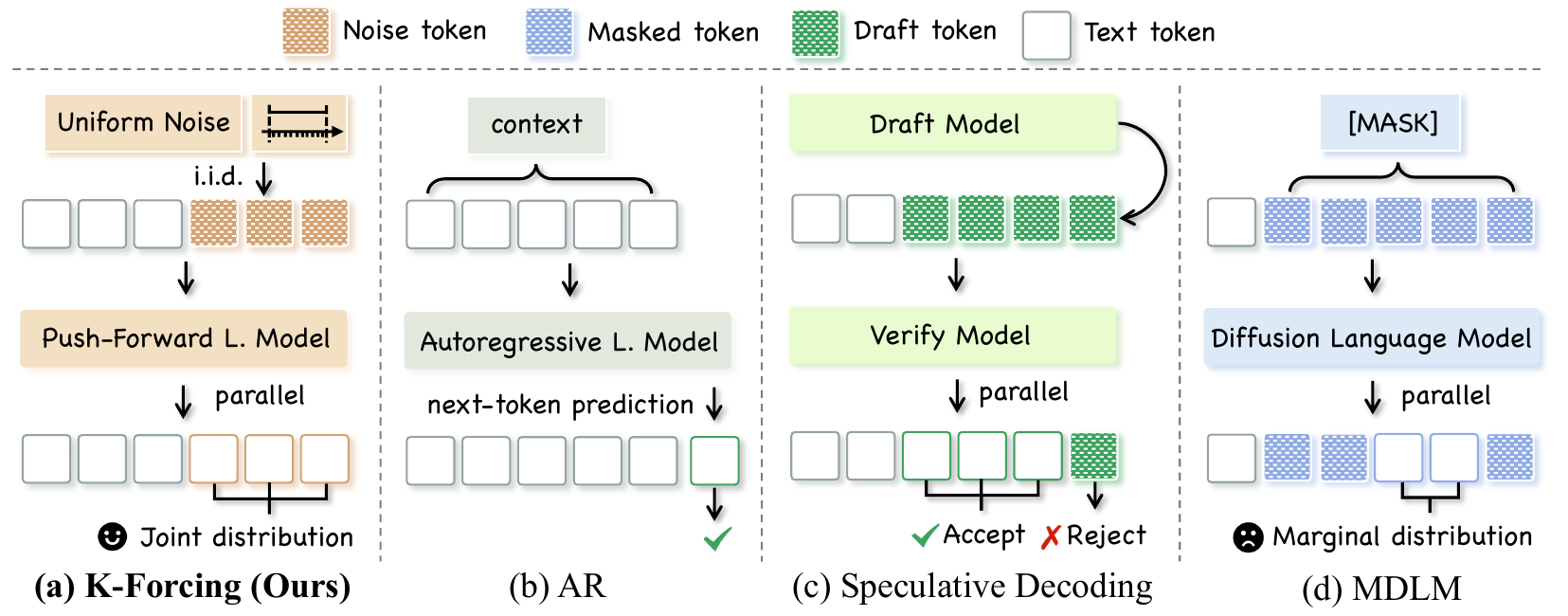}
\caption{Comparison of four language-model inference paradigms within one NFE (number of forward evaluations).
\textbf{(a)}~\textbf{K-Forcing (ours)} uses a push-forward language model to map i.i.d.\ uniform noise tokens to a
\emph{fixed-length} block of future tokens, modeling their \emph{joint}
distribution.
\textbf{(b)}~\textbf{AR} predicts one next token from the current context,
leading to memory-bound decoding.
\textbf{(c)}~\textbf{Speculative decoding} drafts a token block and verifies it
with the target AR model, yielding a variable number of accepted tokens that
breaks regular batching.
\textbf{(d)}~\textbf{MDLM}~\citep{sahoo2024simple,shi2024simplified} predicts masked positions in parallel from
per-position \emph{marginals}, rather than their joint distribution.}
\label{fig:paradigm}
\end{figure}

\section{Introduction}
\label{sec:1:intro}

Large Language Models (LLMs) have demonstrated remarkable capabilities across
a wide range of tasks~\citep{hurst2024gpt,guo2025deepseek,jimenez2023swe},
driven largely by autoregressive (AR) language modeling
\citep{sutskever2014sequence,radford2018improving,brown2020language} with
Transformer architectures~\citep{vaswani2017attention}. AR models represent
sequence distributions by factorizing them into a chain of conditional
distributions and generating tokens one at a time. Although expressive and
effective, this strictly sequential decoding creates a fundamental efficiency
bottleneck: generating an $L$-token sequence requires $L$ forward passes, each
loading model weights and accessing a growing key-value (KV) cache from GPU
memory. Because single-token decoding has low arithmetic intensity, inference
is largely \emph{memory-bound}, leaving modern accelerators underutilized
\citep{kwon2023efficient}. As LLM inference demand grows rapidly,
high-throughput and low-latency serving has become an urgent economic
necessity.

This bottleneck has been addressed at several orthogonal levels. Serving-system
techniques such as prefill--decoding disaggregation~\citep{zhong2024distserve}
and paged attention~\citep{kwon2023efficient} improve hardware utilization and
memory management in large-batch serving. Kernel-level methods and inference
libraries such as FlashAttention~\citep{dao2022flashattention} and
FlashInfer~\citep{ye2025flashinfer} reduce attention cost, while architectural
designs such as state-space models~\citep{gu2024mamba,yang2025gated} and
sparse attention~\citep{yuan2025native} reduce the per-step computation or
memory footprint. However, these approaches mainly improve each decoding step
or the serving system around it; they largely preserve the AR sampling
structure, where each forward pass advances generation by only one token.

At the \emph{statistical modeling level}---how the sequence distribution is
parameterized and sampled---there has been comparatively less progress in
improving high-load batch-serving throughput. Since the core bottleneck is
that each forward pass generates only one token, a natural remedy is to
generate multiple future tokens per pass, amortizing memory-access cost across
several outputs. The closest lines of work are \emph{draft-then-verify}
methods~\citep{chen2023accelerating,sun2023spectr,10.5555/3692070.3693232,kou2024cllms,draxler2025parallel,kumar2026speculative}
and \emph{diffusion language models}~\citep{austin2021structured,10.5555/3692070.3693403,sahoo2024simple,shi2024simplified,nie2025large,ye2025dream}.
Draft-then-verify methods are lossless in principle but mainly reduce
\emph{single-request} latency; their variable-length accepted outputs disrupt
regular batching and can even hurt throughput under heavy load
\citep{kumar2026speculative,liu2026speculativedecodingperformanceillusion}. Diffusion language models decode multiple tokens through iterative denoising,
but their factorized marginal sampling often requires many refinement steps to
preserve quality~\citep{sahoo2024simple,shi2024simplified,nie2025large}, limiting the reduction in forward passes over AR.

To address this gap, we propose \textbf{K-Forcing}, a push-forward language
modeling paradigm for \emph{fixed-length joint next-$k$-token decoding}.
K-Forcing learns a conditional push-forward mapping from a pretrained AR
teacher to generate a joint sample of $k$ future tokens in one forward pass,
hence preserving regular batching. Our contributions are:
\begin{itemize}[leftmargin=*] 
  \item We analyze why existing modeling-level acceleration methods struggle
  in batch serving. Draft-then-verify methods produce variable-length outputs
  that disrupt regular batching, while diffusion language models reveal
  multiple tokens by sampling per-position marginals rather than their joint
  conditional distribution. This motivates the two design principles behind
  K-Forcing: fixed-length outputs and joint multi-token sampling.

\item We formulate K-Forcing as an implicit push-forward generative model for
joint multi-token sampling. We then introduce \emph{progressive self-forcing
distillation} to learn the push-forward mapping from a pretrained AR teacher,
and investigate how architectural choices affect the learning of joint token dependencies.

\item On LM1B and OpenWebText, K-Forcing can generate up to 4 tokens per forward pass and achieves substantial
throughput improvements across low-, medium-, and high-load batch regimes,
with speedups of up to approximately $3\times$ at modest quality degradation relative
to its AR teacher. Generating fewer tokens per pass provides a smooth,
tunable quality--speed trade-off. To our knowledge, this is the first empirical demonstration that a
statistical-modeling change alone can yield substantial batch-serving throughput
gains over AR decoding.
\end{itemize}

\section{Why Existing Approaches Struggle in Batch Serving}
\label{sec:2:analysis}
We analyze two major families of modeling-level approaches to accelerating
AR inference---draft-then-verify methods and diffusion language models---and
identify the challenges they face in improving throughput under high-load
batch-serving settings.

\subsection{Draft-then-Verify Methods}
\label{sec:2:relate:spec}

Draft-then-verify methods, most prominently speculative
decoding~\citep{chen2023accelerating}, accelerate AR inference by using a
lightweight draft model to propose multiple candidate tokens, which are then
verified in parallel by the target model. Because verification uses rejection
sampling against the target distribution, these methods are \emph{lossless} in
principle. The key challenge is to make the draft mechanism both accurate and
cheap. Representative approaches include feature-level drafting (EAGLE~\citep{10.5555/3692070.3693232}), auxiliary multi-token heads
attached to the target model
(Medusa~\citep{10.5555/3692070.3692273}; Hydra~\citep{ankner2024hydra}), and
noise-to-sequence mappings that better match the joint multi-token
distribution for higher acceptance rates~\citep{draxler2025parallel}.

Speculative decoding has proven effective at reducing \emph{single-request
latency} in interactive, low-batch settings, but its variable-length
acceptance pattern poses a fundamental challenge for \emph{throughput-bound
batch serving}. As \cite{zhang2025batch} describe, the \emph{ragged tensor
problem} arises because the number of accepted tokens $a_i$ varies across
requests, desynchronizing position IDs, attention masks, and KV-cache states.
The system must then either pad to the maximum accepted length, wasting
compute, or perform cross-batch realignment~\citep{zhang2025batch}, whose
overhead grows with batch size. Under high serving load, this issue becomes
more pronounced: drafting and verification add extra forward passes, while
variable accepted lengths prevent a proportional reduction in synchronized
decoding steps. \cite{liu2026speculativedecodingperformanceillusion} systematically evaluate
speculative decoding on a production-grade inference engine and confirm that
speedups degrade consistently as batch size grows, because verification of
rejected tokens dominates execution time. As they and
\cite{kumar2026speculative} conclude, speculative decoding is often
ineffective under compute-bound scenarios with large batch sizes, where
throughput gains can be marginal or even negative.

\subsection{Diffusion Language Models and Multi-Token Prediction}
\label{sec:2:relate:diffusion}

Diffusion language models (DLMs)~\citep{austin2021structured,10.5555/3692070.3693403,sahoo2024simple,shi2024simplified,nie2025large,arriola2025block}
have emerged as a promising alternative to AR language models. Among existing
formulations, Masked Diffusion Language Models
(MDLMs)~\citep{sahoo2024simple,shi2024simplified,nie2025large} offer a clean
framework for discrete diffusion over text; the same formulation is referred
to as MD4 by \citet{shi2024simplified}. For brevity, we use the term MDLM
below, and briefly summarize its training objective and inference procedure.

\textbf{Training.}
The MDLM formulation~\citep{sahoo2024simple,shi2024simplified} defines a \emph{forward process} that independently masks each token in a
clean length-$L$ sequence $x_0$ with probability $t \in [0,1]$, producing a
partially masked sequence $x_t$. A mask predictor
$p_\theta(\cdot \mid x_t)$, parameterized by a bidirectional Transformer, is
trained to recover all masked tokens via a cross-entropy loss computed only on
the masked positions:
\begin{align} 
\label{eq:mdlm-objective}
    \mathcal{L}(\theta) \triangleq
    -\mathbb{E}_{t,\,x_0,\,x_t}
    \left[
    \frac{1}{t}\sum_{i=1}^{L}
    \mathbf{1}[x_t^i = \textrm{M}]
    \log p_\theta(x_0^i \mid x_t)
    \right].  
\end{align}

\textbf{Inference.}
Starting from a
fully masked sequence at $t=1$, the generation process is discretized into $T$ steps with
schedule $1=t_T>t_{T-1}>\cdots>t_0=0$. At each step from $t$ to $s<t$, the
mask predictor predicts all masked tokens in parallel, and a fraction $s/t$ of
the predicted tokens are remasked to obtain $x_s$, ensuring consistency with
the forward process. In practice, works such as
LLaDA~\citep{nie2025large} often use confidence-sorted decoding, retaining the
highest-confidence predictions at each step. The number of steps $T$ controls
the quality--efficiency trade-off.

Despite the appeal of parallel prediction, MDLMs~\citep{sahoo2024simple,shi2024simplified} face a fundamental limitation
that directly constrains their ability to reduce NFEs (number of forward
passes). Because the objective in \eqref{eq:mdlm-objective} trains
per-position \emph{marginal} predictors, unmasking multiple positions
simultaneously samples them independently rather than from their joint
conditional distribution. This means that to preserve the target distribution,
MDLMs~\citep{sahoo2024simple,shi2024simplified} must still unmask tokens essentially one at a time---yielding no
reduction in NFEs over AR decoding. We formalize this below.

\begin{theorem}[NFE lower bound for lossless masked-diffusion sampling]
\label{thm:nfe-lower-bound}
Let $p$ be the target data distribution of the clean sequence $x_0$ in
\eqref{eq:mdlm-objective}, supported on
$(x_1,\dots,x_L)\in\mathcal{V}^L$. Suppose $p$ is
\emph{conditionally irreducible}: for any subset $S$ with $|S|\geq2$, any
nontrivial partition $S=S_1\sqcup S_2$, and any $x_{\bar S}$ with
$p(x_{\bar S})>0$, we have
$p(x_S\mid x_{\bar S})\neq
p(x_{S_1}\mid x_{\bar S})p(x_{S_2}\mid x_{\bar S})$, where $\bar S$ is the
complement of $S$. Then, even with a Bayes-optimal mask predictor, i.e., an optimal solution
to \eqref{eq:mdlm-objective}, unmasking $K>1$ tokens in parallel with the
model yields a distribution different from $p$. Hence, lossless sampling under this formulation requires at least $L$ NFEs, i.e., one
token per NFE.
\end{theorem}

We remark that Theorem~\ref{thm:nfe-lower-bound} does not contradict
\cite{jiang2025diffusion}, who show that masked diffusion can reduce NFE relative to AR
for \emph{conditionally reducible} distributions. Moreover,
\cite{nie2025large} show that masked-diffusion language models can achieve
reasonable quality with fewer than $L$ NFEs in practice, suggesting that text
data in certain domains exhibits partial conditional reducibility. The point
of Theorem~\ref{thm:nfe-lower-bound} is to give a \emph{worst-case}
limitation of marginal unmasking under the MDLM formulation~\citep{sahoo2024simple,shi2024simplified}: because the sampler cannot know
a priori which positions are conditionally independent, unmasking multiple
tokens per NFE can deviate from the target joint distribution. We provide the
proof and an illustrative example in Appendix~\ref{app:nfe-lower-bound}.

\textbf{Connection to MTP.}
This marginal-sampling issue is not unique to MDLMs~\citep{sahoo2024simple,shi2024simplified}. Standard multi-token
prediction (MTP) methods~\citep{liu2024deepseek,10.5555/3692070.3692699}
face the same problem: they train auxiliary heads to predict several future
tokens from the same prefix, but the predictions are separate per-position
marginals rather than a joint sample from the future-token block.
Consequently, MTP heads are typically useful as auxiliary training signals but
are discarded at inference, where generation remains autoregressive. Thus,
both MDLMs~\citep{sahoo2024simple,shi2024simplified} and standard MTP expose the same core limitation: \emph{parallel
prediction alone does not provide a reliable joint multi-token sampler.}

In summary, our analysis identifies two requirements for modeling-level
batch-serving acceleration:
\begin{enumerate}[leftmargin=*, itemsep=2pt, topsep=4pt]
    \item \textbf{Fixed-length outputs} per forward pass, to preserve batching regularity;
    \item \textbf{Joint multi-token sampling}, to avoid the degradation caused by independent marginals.
\end{enumerate}
We next introduce K-Forcing, a paradigm designed to satisfy both requirements.


\section{Learning a Push-Forward Language Model with K-Forcing}
\label{sec:3:push}

The requirements of fixed-length outputs and joint multi-token sampling raise
a central question: \emph{how can we build a generative model that produces a
joint sample of $k$ discrete tokens in a single NFE?}

Explicitly modeling this joint distribution is impractical. Given a vocabulary
$\mathcal{V}$, an AR model represents the next-token distribution
$p(x_{t+1}\mid x_{\leq t})$ with a $|\mathcal{V}|$-dimensional logit vector.
Extending this representation to the joint distribution of the next $k$
tokens, $p(x_{t+1},\dots,x_{t+k}\mid x_{\leq t})$, would require a table over
$|\mathcal{V}|^k$ outcomes, which is prohibitively large even for $k{=}2$.

A viable alternative, inspired by implicit generative models such as
GANs~\citep{goodfellow2020generative} and diffusion-style
generators~\citep{song2021scorebased,10.5555/3618408.3619743}, is to learn a
\emph{push-forward mapping} that transforms simple noise into joint token
samples. Specifically, if a noise vector $\mathbf{z}$ is drawn from an
easy-to-sample base distribution $\mu$ and $G$ is a deterministic map, then
the distribution of $G(\mathbf{z})$ is the push-forward of $\mu$ by $G$,
denoted $G_{\#}\mu$~\citep{peyre2019computational}. Learning such a model
means learning $G$ so that $G(\mathbf{z})$ follows the desired target
distribution. 

We instantiate this idea for language modeling and call the resulting framework \emph{K-Forcing}: it learns a push-forward language model that maps $k$ noise variables to the next $k$ tokens in a single forward pass.

\subsection{Formulation of Push-Forward Language Model}
\label{sec:3:push:formulation}

A \emph{push-forward language model} (PFLM) with prediction window $k$
implicitly defines the joint conditional distribution over the next $k$
tokens through a deterministic map
$G_\theta : \mathcal{V}^{t} \times [0,1]^k \to \mathcal{V}^k$.
Given a context $x_{\leq t}$ and $k$ i.i.d.\ noise variables
$\mathbf{z}=(z_1,\dots,z_k)$ with $z_i\sim\mathrm{Uniform}(0,1)$, the map
produces $k$ future tokens in one shot:
$G_\theta(x_{\leq t},\mathbf{z})=(\hat{x}_{t+1},\dots,\hat{x}_{t+k})$.
Unlike an AR language model, PFLM does not explicitly enumerate the joint
likelihood over $\mathcal{V}^k$; instead, one samples from the joint
distribution by drawing $\mathbf{z}$ and evaluating $G_\theta$. 

\textbf{Existence of a push-forward mapping.}
Given access to an autoregressive oracle of the target distribution $p(x)$,
e.g., a well-trained AR model, one can construct a closed-form push-forward
mapping $G^\star$ that maps $k$ independent uniform noise variables to a joint
sample of the next $k$ tokens. Let $q_{\mathrm{AR}}(\cdot\mid x_{\leq t})$
denote the oracle's next-token distribution given context $x_{\leq t}$, and let
$F_{\mathrm{AR}}(v \mid x_{\leq t}) = \sum_{v' \leq v} q_{\mathrm{AR}}(v' \mid x_{\leq t})$
be its cumulative distribution function (CDF). The inverse CDF (quantile function)
$F_{\mathrm{AR}}^{-1}(\cdot \mid x_{\leq t})$ maps a uniform random variable
$z \sim \mathcal{U}[0,1)$ to a token $v$ such that
$F_{\mathrm{AR}}^{-1}(z \mid x_{\leq t}) = \min\{v : F_{\mathrm{AR}}(v \mid x_{\leq t}) \geq z\}$.
Given the noise vector $\mathbf{z}$, we recursively generate the next $k$ tokens as
\begin{equation}
  \hat{x}_{t+j}
  =
  F_{\mathrm{AR}}^{-1}
  \bigl(z_j \mid x_{\leq t}, \hat{x}_{t+1:t+j-1}\bigr),
  \qquad j=1,\dots,k .
  \label{eq:main-sequential-icdf}
\end{equation}

Unrolling this recursion gives the closed-form map
$G^\star(x_{\leq t},\mathbf{z})=(\hat{x}_{t+1},\dots,\hat{x}_{t+k})$.
This shows that the PFLM formulation is expressive enough, in principle, to
reproduce the joint conditional distribution of an AR teacher without
explicitly parameterizing the $|\mathcal{V}|^k$ joint probability table.
For completeness, we provide a rigorous analysis showing that
\eqref{eq:main-sequential-icdf} is equivalent to AR sampling, together with
an illustrative example, in Appendix~\ref{app:existence}.

We remark that this noise-conditioned formulation and the inverse-CDF construction above are equivalent to O-PTP and the construction in Theorem~1 of PTP~\citep{draxler2025parallel}, respectively; we use push-forward notation to develop our progressive self-forcing distillation method.

\textbf{Comparison with existing inference paradigms.}
Figure~\ref{fig:paradigm} summarizes the contrast between K-Forcing, which trains a PFLM, and existing
paradigms within a single NFE. AR advances one token at a time, speculative
decoding yields variable-length accepted outputs, and MDLM~\citep{sahoo2024simple,shi2024simplified} samples multiple
positions from per-position marginals. In contrast, K-Forcing uses PFLM to map i.i.d.\ noise to
a fixed-length joint sample of $k$ future tokens.

\subsection{Supervision Strategy}
\label{sec:3:push:training}

We have shown that a closed-form push-forward mapping can be obtained by
unrolling the AR sampling process in \eqref{eq:main-sequential-icdf}.
Therefore, a natural way to learn such a map in practice is
\emph{distillation}: we use the AR teacher to construct supervised
noise--token pairs $(\mathbf{z},\hat{\mathbf{x}})$, where each noise vector
$\mathbf{z}$ is paired with a teacher-generated token block
$\hat{\mathbf{x}}$. The PFLM student is then trained to match this mapping by
minimizing the discrepancy between its prediction
$G_\theta(x_{\leq t},\mathbf{z})$ and the teacher target
$\hat{\mathbf{x}}$.

For each training position $t\in\mathcal{T}$, let
$\mathbf{z}^{(t)}=(z^{(t)}_1,\dots,z^{(t)}_k)$ denote the noise vector
associated with context $x_{\leq t}$. We parameterize the student to
produce categorical distributions
$p_{\theta,j}(\cdot \mid x_{\leq t},\mathbf{z}^{(t)})$ for
$j=1,\dots,k$, where each distribution predicts the $j$-th future token.
These distributions parameterize the map $G_\theta$: at inference time, the
$j$-th output token is obtained by greedy decoding from
$p_{\theta,j}(\cdot \mid x_{\leq t},\mathbf{z}^{(t)})$. In this way, we can train
the student with the \emph{next-$k$-token prediction loss}, using the
standard negative log-likelihood (NLL) averaged over all training positions in
$\mathcal{T}$ and future offsets:
\begin{equation}
  \mathcal{L}_{\mathrm{PFLM}}(\theta)
  =
  -\frac{1}{|\mathcal{T}|k}
  \sum_{t\in\mathcal{T}}\sum_{j=1}^{k}
  \log p_{\theta,j}\!\bigl(
    \hat{x}_{t+j} \mid x_{\leq t},\,\mathbf{z}^{(t)}
  \bigr).
  \label{eq:pflm-loss}
\end{equation}

Importantly, under the inverse-CDF construction in
\eqref{eq:main-sequential-icdf}, the teacher-generated target block
$\hat{\mathbf{x}}$ is deterministic once the context $x_{\leq t}$ and
noise vector $\mathbf{z}^{(t)}$ are fixed. Therefore, if the noise--token
pairs are constructed ideally and the student has sufficient capacity, the
\emph{next-$k$-token prediction loss} in \eqref{eq:pflm-loss} can in principle be reduced arbitrarily close to zero.

The key question is therefore:
\emph{how can we construct noise--token pairs
$(\mathbf{z},\hat{\mathbf{x}})$ from an AR teacher as ideally as possible?}

\subsubsection{Baseline: Noise Inversion}
\label{sec:3:push:training:baseline}

A natural baseline is \emph{noise inversion}, proposed by
\cite{draxler2025parallel} to train a noise-conditioned multi-token drafter
for speculative decoding. Given a real-data sequence $(x_1,\dots,x_T)$, noise
inversion recovers, for each token $x_{t+1}$, a noise value
$z_t^*\in[0,1]$ that would generate this token under inverse-CDF sampling from
the AR teacher. Specifically, let the CDF bin of token $x_{t+1}$ be
$[l_t,u_t)$, where
$l_t=\sum_{v < x_{t+1}} q_{\mathrm{AR}}(v\mid x_{\leq t})$ and
$u_t=l_t+q_{\mathrm{AR}}(x_{t+1}\mid x_{\leq t})$. Noise inversion samples
$z_t^*\sim\mathrm{Uniform}(l_t,u_t)$, so applying the teacher's inverse-CDF
sampler to $z_t^*$ recovers $x_{t+1}$. In this way, a single AR teacher
forward pass over the real-data sequence constructs paired noise--token
examples $(z_t^*,x_{t+1})$ for all positions $t$, making noise inversion
convenient and efficient.

Although simple, noise inversion has two critical limitations:

\textbf{Train--inference mismatch.}
The recovered noise is uniform only when the target tokens are sampled from the AR teacher. With real training data, tokens may lie in low-probability regions of the teacher distribution, so the inverted noise concentrates near the edges of narrow CDF bins rather than spreading uniformly over $[0,1]$. Since the PFLM student receives genuinely uniform noise at inference time, this creates a systematic train--inference mismatch that degrades generation quality.

\textbf{Numerical fragility.}
A more serious issue is that noise inversion is numerically fragile. For low-probability tokens, the corresponding CDF bin $[l_t, u_t)$ is extremely narrow, meaning that even tiny perturbations in the teacher logits or in the cumulative sum used to compute CDF boundaries can shift the recovered noise into an adjacent bin, producing a different token entirely. Such perturbations are difficult to avoid in modern GPU execution: cuBLAS does not guarantee bitwise reproducibility across different batch sizes~\citep{deepseekv4}, and attention kernels may use non-deterministic floating-point accumulation orders~\citep{he2025nondeterminism,deepseekv4}. In practice, we observe that reapplying inverse-CDF sampling to the inverted noise often fails to recover the original token, confirming that the round-trip is not reliably invertible. These issues explain why \cite{draxler2025parallel} rely on speculative verification to filter incorrect predictions; consequently, their method inherits the variable-length decoding and irregular batching challenges of draft-then-verify approaches.

\subsubsection{K-Forcing: Progressive Self-Forcing Distillation}
\label{sec:3:push:training:selfforcing}

To address the limitations of noise inversion, we propose K-Forcing, which trains PFLM via \emph{progressive
self-forcing distillation}. This approach is inspired by self-forcing techniques for training
autoregressive video diffusion models~\citep{huang2025self}, which perform autoregressive self-rollout for causal video models during training so that the model learns from its own predictions rather than ground-truth context.

\textbf{Stage 1: Forward Distillation (AR $\to$ PFLM($k=1$)).}
Rather than inverting ground-truth tokens back into noise, we run the AR
teacher \emph{forward}. For each context $x_{\leq t}$, we sample
$z\sim\mathrm{Uniform}(0,1)$ and use the teacher's inverse-CDF sampler to
generate
$\hat{x}_{t+1}=F_{\mathrm{AR}}^{-1}(z\mid x_{\leq t})$.
The PFLM($k{=}1$) student is trained to predict $\hat{x}_{t+1}$ from
$(x_{\leq t},z)$. This construction removes the train--inference mismatch
because the noise is uniform by construction, and avoids recovering noise from
narrow CDF bins. We use this forward distillation stage only to bootstrap a
reliable PFLM with $k{=}1$.

\textbf{Stage 2: Self-forcing Distillation (PFLM($k$) $\to$ PFLM($2k$)).}
Once a PFLM with window $k$ is trained, we use it as the teacher to distill a
new PFLM student with window $2k$. Given context $x_{\leq t}$ and noise
$\mathbf{z}=(z_1,\dots,z_{2k})$, the PFLM($k$) teacher first generates
$(\hat{x}_{t+1},\dots,\hat{x}_{t+k})$ in one forward pass using
$(z_1,\dots,z_k)$. In a second forward pass, it conditions on the extended
context $(x_{\leq t},\hat{x}_{t+1},\dots,\hat{x}_{t+k})$ and the remaining
noise $(z_{k+1},\dots,z_{2k})$ to generate
$(\hat{x}_{t+k+1},\dots,\hat{x}_{t+2k})$. The PFLM($2k$) student is then
trained with the full noise vector $\mathbf{z}$ as input and the concatenated
sequence $(\hat{x}_{t+1},\dots,\hat{x}_{t+2k})$ as its target.

\textbf{Progressive window expansion.}
Stage 2 is applied repeatedly: starting from the bootstrapped PFLM($k{=}1$), we
run self-forcing distillation with $k=1\to2\to4$, progressively doubling the
prediction window at each stage. This enables K-Forcing to scale to large $k$
while maintaining stable supervision throughout the distillation chain.

Unlike noise inversion, self-forcing directly addresses both limitations
identified above. First, because the noise is always sampled from
$\mathrm{Uniform}(0,1)$ and never inverted from real data, there is no
train--inference distribution mismatch. More importantly, self-forcing
greatly alleviates numerical fragility after the bootstrap stage: the teacher
maps noise to tokens through its learned push-forward network without requiring inversion through narrow CDF bins, so the supervision is robust to floating-point
non-determinism.

\begin{figure}[htbp]
  \centering
  \includegraphics[width=0.7\textwidth]{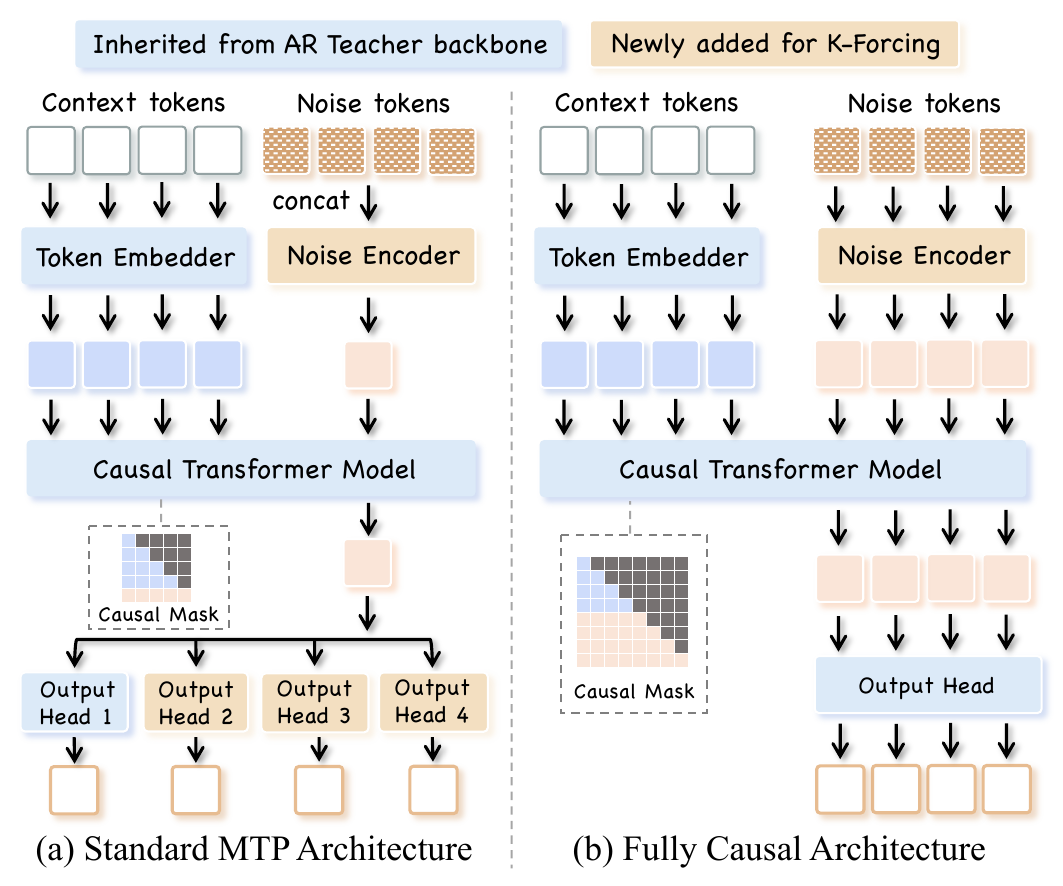}
  \caption{
  \textbf{(a) Standard MTP}: $k$ noise variables are concatenated into one
  token; $k$ independent heads decode future tokens from its hidden state.
  \textbf{(b) Fully causal}: each noise variable forms a separate token under
  causal attention and is decoded by a shared head. Both reuse the AR backbone
  with a sinusoidal+MLP noise encoder.}
  \label{fig:architecture}
\end{figure}

\subsection{Architectural Choices}
\label{sec:3:push:arch}

We consider two architectures for parameterizing the push-forward mapping
$G_\theta$, as illustrated in Figure~\ref{fig:architecture}. Both designs
reuse the pretrained AR backbone and inject noise through a sinusoidal+MLP
encoder, differing only in how noise variables are organized within the
sequence.

\textbf{Standard MTP architecture.}
As a natural baseline, we extend the standard MTP-style
architecture~\citep{liu2024deepseek,10.5555/3692070.3692699} by concatenating
all $k$ noise variables into a single noise-conditioning token. From the
hidden state of this token, $k$ independent prediction heads decode the
future tokens in parallel. This approach is simple to implement, requiring only a noise encoder and $k$ independent prediction heads on top of the shared backbone.

\textbf{Fully causal architecture.}
We consider a \emph{fully causal} architecture that represents each
noise variable as a separate token under causal attention, decoded by a shared
prediction head. This architecture is also used in PTP~\citep{draxler2025parallel}.
This design mirrors the inverse-CDF construction in
\eqref{eq:main-sequential-icdf}: the $j$-th future token depends only on the
context and prefix noise variables $(z_1,\dots,z_j)$, not on future noise
$(z_{j+1},\dots,z_k)$. Each future token is thus conditioned on its own
dedicated noise variable, avoiding the information bottleneck of routing all
$k$ stochastic decisions through a single shared latent. This causal structure
provides a natural inductive bias for learning the push-forward mapping and
allows the model to reuse the AR backbone without additional MTP heads.

\subsection{Practical Considerations}
\label{sec:3:push:impl}

\textbf{Compatibility with AR serving infrastructure.} K-Forcing preserves the fixed-length, synchronized structure required by KV-cache reuse and continuous batching. The inference procedure mirrors standard AR KV-cache decoding: at each step, the model consumes a fixed number of input tokens, produces exactly $k$ output tokens, and appends their KV entries to the cache. This regular, fixed-stride structure ensures that all requests in a batch remain synchronized in position indices and cache lengths, making the scheme directly compatible with continuous batching schedulers without cross-request realignment or padding. The prediction window can also be varied at inference time, using any $k \leq k_{\mathrm{train}}$, to trade off speed and quality without retraining. We provide the full inference algorithm in Appendix~\ref{app:3:implementation:kvcache}. 

\textbf{Training cost.} The main training cost comes from self-forcing distillation, which requires two sequential teacher forward passes per context position to produce the $2k$-token targets, plus one student forward pass. To implement this efficiently, each of the two consecutive teacher passes is realized as a batched attention call with structured attention masks that encode the causal dependencies between context, predicted, and noise tokens within a single forward pass. Theoretically, these masks are block-sparse with $O(k)$ non-zero entries relative to standard AR attention when $N \gg k$, so training costs $O(k)\times$ a single AR forward pass. However, our current implementation does not yet exploit this sparsity and passes the masks as dense matrices to FlashAttention~\citep{dao2023flashattention2}, resulting in $O(k^2)\times$ AR cost in practice. Developing custom kernels that exploit the block-sparse structure of the masks to close this gap is left for future work; we discuss the required design in Appendix~\ref{app:3:implementation:training}. 

\textbf{Temperature-controlled generation.} K-Forcing can be extended  to support temperature control by augmenting the mapping as $G_\theta(x_{\leq t},\mathbf{z},\tau)$. During training, a temperature $\tau$ is sampled uniformly from $[\tau_{\min}, \tau_{\max}]$ and used both to sharpen or flatten the teacher output distribution and to condition the noise encoder, which learns to associate different $\tau$ values with different diversity levels. At inference, varying $\tau$ smoothly interpolates between near-greedy outputs (low $\tau$) and diverse samples (high $\tau$)---no retraining or model modification is required. We detail this extension in Appendix~\ref{app:3:implementation:training}.  


\section{Experiments}
\label{sec:4:exp} 

We evaluate K-Forcing through three main experiments: its quality--throughput
trade-off in batch serving (Section~\ref{sec:4:exp:quality}), ablations of
the supervision strategy and  architecture design
(Section~\ref{sec:4:exp:ablation}), and a quality--NFE comparison with
existing language-model inference paradigms
(Section~\ref{sec:4:exp:bz1}).
Additional results on temperature-controlled generation are provided in
Appendix~\ref{app:4:experiment:temperature}.

\textbf{Datasets and architecture.}
We evaluate K-Forcing on two standard language modeling benchmarks, LM1B
\citep{chelba2013one} and OpenWebText (OWT)
\citep{gokaslan2019openwebtext}. Our setup largely follows MDLM
\citep{sahoo2024simple}: we use the same dataset preprocessing, tokenizers,
context lengths, and Transformer backbone. Specifically, the context length is
128 for LM1B and 1{,}024 for OWT, and all models use the same 12-layer,
$\sim$100M-parameter Transformer backbone. For the AR teacher, we use the
 AR checkpoint on OWT released by \citep{sahoo2024simple}, trained for 1M steps with  batch
size 512, and train our own AR teacher on LM1B for 500K steps with the same
 batch size. Remaining hyperparameters and implementation details are
provided in Appendix~\ref{app:4:experiment:config}.

\textbf{K-Forcing training.}
By default, K-Forcing trains a PFLM with the fully causal architecture using
progressive self-forcing distillation, following the sequence
AR $\to$ PFLM($k{=}1$) $\to$ PFLM($k{=}2$) $\to$ PFLM($k{=}4$).
At each stage, the student is initialized from the previous-stage teacher and
then trained on the same dataset for 500K steps with batch size 512. All
training experiments are conducted on a single pod with 8 H100 GPUs. For
distillation and sampling-based evaluation, we take the AR teacher
distribution at temperature $\tau=1.0$ as the target distribution.

\subsection{Joint Multi-Token Prediction for Batch Serving}
\label{sec:4:exp:quality} 

We first evaluate whether K-Forcing can produce an effective joint multi-token
predictor. Since the resulting PFLM is trained to match the joint future-token distribution
of its AR teacher, comparable generation quality to AR would indicate that
the learned push-forward mapping captures the desired joint distribution. Using
the default training recipe described above, we evaluate the final PFLM($k{=}4$) checkpoint at
$k\in\{2,3,4\}$ by varying the number of input noise tokens at decoding time.

\textbf{Inference protocol.}
We evaluate batch-serving throughput on fixed-length completion tasks. 
For each dataset, we sample 1{,}024 held-out prefixes, completing 6-token 
prefixes to 128 tokens on LM1B and 64-token prefixes to 1{,}024 tokens on OWT. 
Throughput is measured under bf16 precision for both AR and K-Forcing with KV-cache 
enabled, at batch sizes of 4, 16, and 128 on a single H100 GPU, corresponding 
to low-, medium-, and high-load regimes; the largest size saturates GPU 
utilization in our setup. We report the total number of \emph{generated tokens} 
(including special tokens) divided by wall-clock decoding time, averaged over 
5 runs after one warm-up. Attention is computed via the FlashAttention-v2 
kernel~\citep{dao2023flashattention2}.

\textbf{Generation-quality metrics.}
Since K-Forcing is an implicit sampler rather than an explicit likelihood model,
we evaluate generation quality using sample-based metrics. First, we report
\emph{generative perplexity} (Gen-PPL), computed by scoring the generated
completions with an external GPT-2-large evaluator~\citep{radford2019language}. 
For LM1B, generated samples are re-tokenized before GPT-2-large scoring
because of different tokenizers. \emph{To ensure fair quality assessment, 
we truncate each generated sequence at the first end-of-sequence token 
before 
computing the metrics.} 
Second, we conduct an \emph{LLM-as-a-judge} evaluation as an affordable
proxy for human preference assessment. Using the same prefixes, we use a
locally-served Qwen3.5-27B~\citep{qwen35blog} to perform pairwise comparisons
between each \emph{truncated} K-Forcing completion and the AR completion generated
from the same prefix; we use a local model for full reproducibility. 
The judge is forced to choose one completion based on coherence, fluency,
and naturalness, and we report the win rate of each K-Forcing variant against AR.
The full judge prompt and evaluation protocol details are provided in
Appendix~\ref{app:4:experiment:eval}.

\begin{table}[t]
\centering
\caption{Quality--throughput trade-off for K-Forcing on LM1B and OWT.
K-Forcing($k$) denotes decoding with $k$ noise tokens from the PFLM($k{=}4$) checkpoint trained by K-Forcing.
Throughput is reported in \textbf{k tokens/s}, with speedup over AR shown in
parentheses.}
\label{tab:quality_throughput}
\begin{tabular}{l l c c c c c}
\toprule
& & \multicolumn{2}{c}{Quality} & \multicolumn{3}{c}{Throughput (k tokens/s, speedup)} \\
\cmidrule(lr){3-4}\cmidrule(lr){5-7}
Dataset & Method & Gen-PPL\,$\downarrow$ & Win\,vs.\,AR\,$\uparrow$ & BS=4 & BS=16 & BS=128 \\
\midrule
\multirow{4}{*}{LM1B}
& AR             & 104.8 & -- & 0.54\;($1.00\times$) & 2.03\;($1.00\times$) & 15.4\;($1.00\times$) \\
& K-Forcing ($k{=}2$) & 107.2 & 50.2\% & 0.84\;($1.56\times$) & 3.48\;($1.71\times$) & 24.7\;($1.60\times$) \\
& K-Forcing ($k{=}3$) & 117.1 & 45.1\% & 1.27\;($2.35\times$) & 4.99\;($2.46\times$) & 35.6\;($2.31\times$) \\
& K-Forcing ($k{=}4$) & 127.6 & 42.9\% & 1.69\;($3.13\times$) & 6.77\;($3.33\times$) & 46.5\;($3.02\times$) \\
\midrule
\multirow{4}{*}{OWT}
& AR             & 42.64 & -- & 0.53\;($1.00\times$) & 1.99\;($1.00\times$) & 9.61\;($1.00\times$) \\
& K-Forcing ($k{=}2$) & 32.82 & 46.9\% & 0.85\;($1.60\times$) & 3.36\;($1.69\times$) & 11.9\;($1.24\times$) \\
& K-Forcing ($k{=}3$) & 29.67 & 42.8\% & 1.28\;($2.42\times$) & 5.03\;($2.53\times$) & 17.5\;($1.82\times$) \\
& K-Forcing ($k{=}4$) & 24.97 & 39.4\% & 1.70\;($3.21\times$) & 6.91\;($3.47\times$) & 22.9\;($2.38\times$) \\
\bottomrule
\end{tabular} 
\end{table}

\textbf{Results.}
Table~\ref{tab:quality_throughput} shows that K-Forcing achieves large
batch-serving speedups with moderate quality degradation relative to the AR
teacher. K-Forcing($k{=}4$) reaches roughly $3\times$ speedup on LM1B and
$2.4$--$3.5\times$ speedup on OWT across different batch sizes,
showing consistent gains from latency-bound to compute-bound regimes.
Varying $k$ yields a smooth quality--throughput trade-off:
K-Forcing($k{=}2$) is nearly indistinguishable from AR (win rate $50.2\%$
on LM1B, $46.9\%$ on OWT) at $\sim\!1.6\times$ speedup, while
K-Forcing($k{=}3$) sits in between at $\sim\!2.4\times$ with win rates
above $42\%$. This lets practitioners pick $k$ to match their
latency--quality budget. The OWT Gen-PPL results also highlight why
the LLM-as-a-judge metric matters: K-Forcing reports a \emph{lower}
Gen-PPL than the AR teacher, but
this simply reflects GPT-2-large finding the push-forward sampler's
output more ``typical''---not that K-Forcing actually surpasses the teacher.
The pairwise judge, by directly comparing matched completions, gives
a more reliable quality signal. Overall, K-Forcing offers a favorable,
$k$-controllable quality--speed trade-off, making it an effective
joint multi-token predictor for batch serving.

\subsection{Ablation Analysis}
\label{sec:4:exp:ablation} 

We ablate  the supervision strategy
used to construct noise--token pairs
(Section~\ref{sec:3:push:training}) and the architecture used to
parameterize the push-forward mapping
(Section~\ref{sec:3:push:arch}).

\begin{table}[htbp]
  \centering
\caption{Per-position validation NLL after 200K distillation steps for
K-Forcing($k{=}4$). Lower is better. Variants: (A) noise inversion + standard MTP;
(B) noise inversion + fully causal; and (C) self-forcing + fully causal. The
AR row reports the converged AR teacher validation NLL.}
  \label{tab:ablation_final_nll}
\begin{tabular}{ccccccc}
  \toprule
  Dataset & Variant
  & $\mathcal{L}_{1}$ & $\mathcal{L}_{2}$
  & $\mathcal{L}_{3}$ & $\mathcal{L}_{4}$ & Avg. \\
  \midrule
  \multirow{4}{*}{LM1B}
     & AR & 3.04 & -- & -- & -- & -- \\
  & (A) & 2.84 & 4.43 & 5.38 & 5.90 & 4.64 \\
  & (B) & 2.51 & 3.94 & 4.82 & 5.34 & 4.15 \\
  & (C) & \textbf{0.87} & \textbf{1.90} & \textbf{2.73} & \textbf{2.87} & \textbf{2.09} \\
  \midrule
  \multirow{4}{*}{OWT}
    & AR & 2.86 & -- & -- & -- & -- \\
  & (A) & 2.27 & 4.31 & 5.73 & 6.10 & 4.60 \\
  & (B) & 1.87 & 3.56 & 4.52 & 4.95 & 3.73 \\
  & (C) & \textbf{0.39} & \textbf{2.05} & \textbf{2.82} & \textbf{2.92} & \textbf{2.05} \\
  \bottomrule
\end{tabular}
 
\end{table}

\textbf{Variants.}
We compare three K-Forcing($k{=}4$) variants:
\textbf{(A)} noise inversion + standard MTP,
\textbf{(B)} noise inversion + fully causal, and
\textbf{(C)} self-forcing distillation + fully causal.
All variants are initialized from the same AR checkpoint for a fair
comparison. For noise inversion, supervision is constructed from the AR
teacher. For self-forcing distillation, supervision is generated by a
well-trained K-Forcing($k{=}2$) teacher.

\textbf{Metric.}
We use per-position NLL on the validation sets of LM1B and OWT during the
first 200K distillation steps to compare the variants. Specifically, for each
future position $j\in\{1,\dots,4\}$, we define
$\mathcal{L}_j(\theta)=-\sum_{t\in\mathcal{T}}
\log p_{\theta,j}(\hat{x}_{t+j}\mid x_{\leq t},\mathbf{z}^{(t)})/|\mathcal{T}|$,
which is the $j$-th per-position component of the K-Forcing objective in
\eqref{eq:pflm-loss}. This metric measures how accurately the learned
push-forward mapping predicts the $j$-th future token given the prefix and input
noise; lower values indicate that the student more faithfully recovers the
underlying AR teacher's noise-to-token mapping. We also report the final
converged validation NLL of the AR teacher as a context-only next-token
prediction reference. \emph{Achieving lower NLL than this AR reference suggests
that, conditioned on both the context and input noise, a future token can
be predicted nearly as accurately as the next token predicted by an AR model
from context alone.}

\begin{figure}[htbp]
  \centering
  \includegraphics[width=0.6\textwidth]{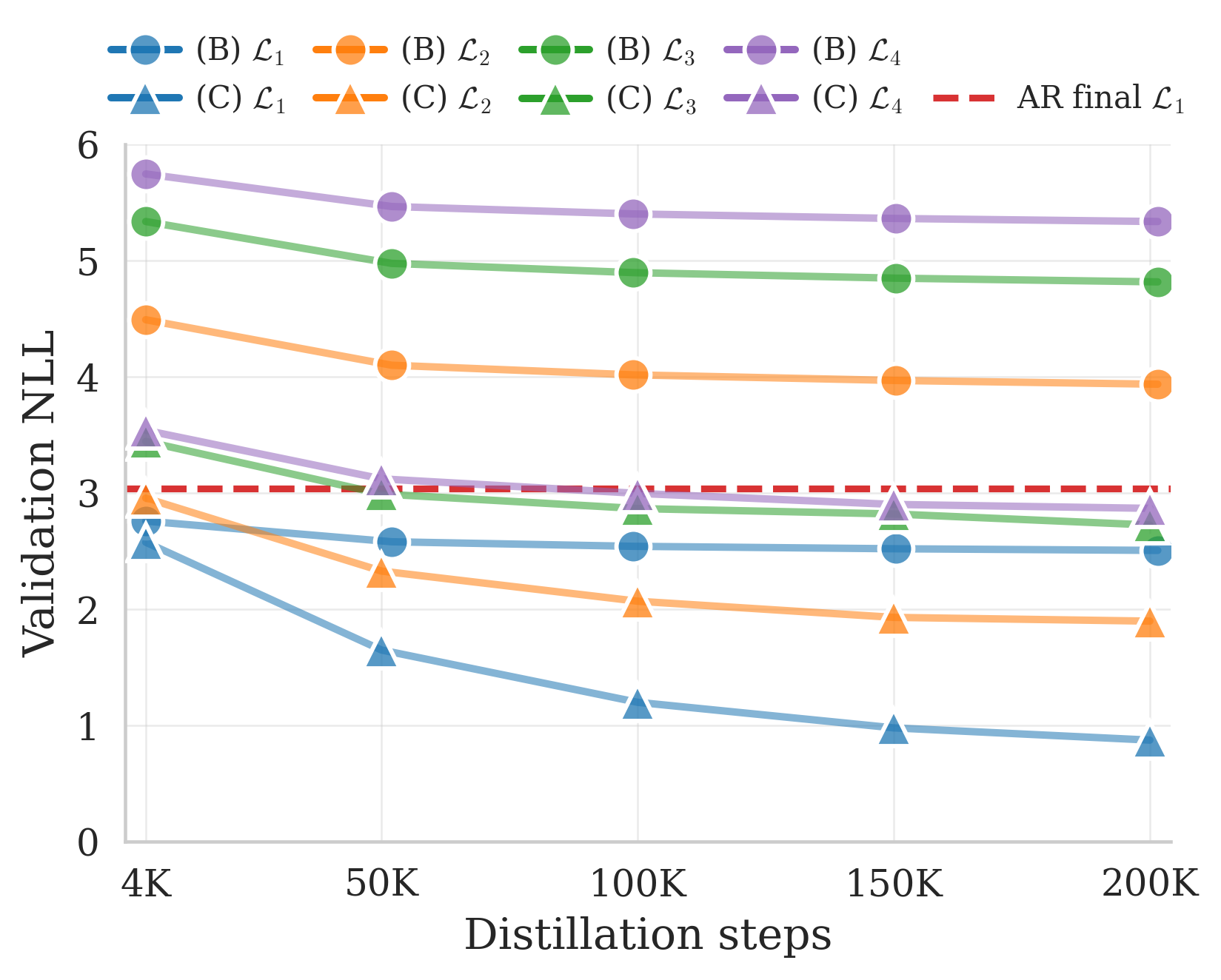}
  \caption{Training dynamics of K-Forcing on LM1B during the first 200K
  distillation steps. We plot per-position validation NLL for variants
  (B) and (C), together with the AR teacher reference line.}
  \label{fig:ablation_lm1b_dynamics}
\end{figure}

\textbf{Results.}
Table~\ref{tab:ablation_final_nll} shows that the default K-Forcing design
\textbf{(C)} achieves
the lowest validation NLL on both datasets and across all prediction
positions, improving over \textbf{(A)} and \textbf{(B)} by a large margin. The comparison between \textbf{(A)} and \textbf{(B)}
indicates that the fully causal architecture improves performance under the
same supervision strategy, suggesting that leveraging this inductive bias is
helpful. Meanwhile, the significant gap between \textbf{(B)} and \textbf{(C)}
highlights the importance of a precise teacher supervision signal. More
importantly, \textbf{(C)} is the only variant that achieves lower validation
NLL than the AR teacher reference at almost all positions after only 200K
distillation steps.

\textbf{Training dynamics of K-Forcing.}
Figure~\ref{fig:ablation_lm1b_dynamics} shows that \textbf{(C)} converges much
faster and to a lower validation NLL than both \textbf{(A)} and \textbf{(B)},
suggesting that stable self-forced teacher supervision also accelerates
convergence. We also observe that $\mathcal{L}_j$ converges more slowly as
$j$ increases, suggesting that farther future positions are harder to learn.
Even for the best variant \textbf{(C)}, however, the per-position NLL is not
driven to zero at convergence, indicating that the learned push-forward mapping
remains imperfect. This may be due in part to limited student capacity, since
the current model scale is relatively small. A more plausible explanation is
that the self-forced supervision signal is still imperfect, since batch-variant
GPU kernel effects can still slightly perturb the K-Forcing teacher outputs.

\subsection{Comparison with Existing Inference Paradigms}
\label{sec:4:exp:bz1} 

We further compare K-Forcing with the two modeling-level acceleration paradigms
discussed in Section~\ref{sec:2:analysis}. We instantiate these paradigms
with MDLM~\citep{sahoo2024simple,shi2024simplified} and
Medusa~\citep{10.5555/3692070.3692273}. Medusa performs speculative decoding
using lightweight MTP heads attached to the AR model, making it a suitable
baseline in our setting; stronger variants such as
EAGLE~\citep{10.5555/3692070.3693232} and
Hydra~\citep{ankner2024hydra} introduce additional autoregressive modules
whose inference cost is non-negligible at our model scale. We also include PTP~\citep{draxler2025parallel}, which is essentially
equivalent to using variant~\textbf{(B)} from Section~\ref{sec:4:exp:ablation}
as the drafting model in speculative decoding.  As discussed in
Section~\ref{sec:2:analysis}, these paradigms have different serving
structures, so raw batch-serving throughput is not a suitable metric for this
comparison. We therefore compare them using the quality--NFE trade-off.

\textbf{Comparison design.}
To ensure a consistent comparison across paradigms, we keep all methods on
the same Transformer backbone as AR and define NFE in a comparable way.
For AR, MDLM, and K-Forcing, one NFE corresponds to a single model forward pass.
For Medusa and PTP, each speculative decoding iteration requires two forward
passes---one draft pass to propose candidate tokens and one target AR pass to
verify them---so we report their NFE as $2\times$ the number of speculative
iterations (see Table~\ref{tab:nfe_quality}).
We remark that our PTP implementation is simplified for convenience, whereas the original Partial Quadratic Decoding overlaps drafting and verification~\citep{draxler2025parallel}; the $2\times$ factor is therefore only a reference for our implementation.
Note that this NFE definition is conservative for K-Forcing: an MDLM forward pass
uses bidirectional attention, and both Medusa and PTP require separate draft
and target-model computation, whereas K-Forcing uses a single causal forward pass
to produce a fixed number of tokens.

\textbf{Experiment setup.}
We conduct this study on the OWT dataset using the same inference protocol as
in Section~\ref{sec:4:exp:quality}: we sample 1024 held-out prefixes of
64 tokens, generate completions to length 1024, and evaluate generation
quality with the same sample-based metrics, Gen-PPL and win rate against AR.
For AR and K-Forcing, we use the same checkpoints and setup as in
Section~\ref{sec:4:exp:quality}.
For MDLM, we use the checkpoint released by \citep{sahoo2024simple} and
reduce NFE by unmasking a fixed number of $k$ tokens per forward pass,
selecting the top-$k$ positions by confidence after sampling at temperature
$\tau{=}1.0$, following the practice in LLaDA~\citep{nie2025large}.
For Medusa, we attach $k{=}4$ MTP heads to the AR backbone and train them on
the same OWT training set for 500K steps with batch size 512.
For PTP, we reuse the variant~\textbf{(B)} checkpoint from
Section~\ref{sec:4:exp:ablation} as the draft model.

\begin{table}[htbp]
\centering
\caption{Quality--NFE trade-off on OWT. We report the number of forward evaluations (NFE) required to generate a 960-token completion, Gen-PPL (lower is better), and LLM-as-a-Judge Win rate against AR (higher is better). For speculative methods (Medusa, PTP), NFE counts both draft and verification passes.}
\label{tab:nfe_quality}
\begin{tabular}{l c c c}
\toprule
Method & NFE\,$\downarrow$ & Gen-PPL\,$\downarrow$ & Win vs.\ AR\,$\uparrow$ \\
\midrule
AR      & 960 & 42.64& --     \\
\midrule
MDLM ($k{=}1$)    & 960 & 65.17 & 42.4\% \\
MDLM ($k{=}2$)    & 480  & 224.1 & 26.6\% \\
\midrule
Medusa ($k{=}4$)  &  $2\times$539.3 & 44.21 & 50.3\% \\
PTP ($k{=}4$)  &  $2\times$339.2 & 41.98 & 49.6\% \\
\midrule
K-Forcing ($k{=}2$)    & 480  & 32.82 & 46.9\% \\
K-Forcing ($k{=}3$)    & 320  & 29.67 & 42.8\% \\
K-Forcing ($k{=}4$)    & 240  & 24.97 & 39.4\% \\
\bottomrule
\end{tabular}
\end{table}

\textbf{Results.} Table~\ref{tab:nfe_quality} shows that K-Forcing achieves the most favorable
quality--NFE trade-off among the compared paradigms.
MDLM already underperforms AR at the same 960 NFEs (65.17 vs.\ 42.64
Gen-PPL), and degrades sharply when unmasking two tokens per step (224.1
Gen-PPL at 480 NFEs), consistent with Theorem~\ref{thm:nfe-lower-bound}.
Medusa and PTP preserve AR-level quality through speculative verification,
but each iteration requires two full-size forward passes; Medusa ends up
using $2\times 539.3\!\approx\!1079$ NFEs---\emph{more} than AR---while our PTP
implementation requires 339.2 speculative iterations. Even without the $2\times$
factor, K-Forcing ($k=4$) uses fewer forward passes (240) and achieves lower
Gen-PPL (24.97 vs.\ 41.98).
In contrast, K-Forcing reduces NFE directly via a single causal forward pass:
K-Forcing($k{=}2$) matches the 480-NFE budget of MDLM($k{=}2$) with far better
Gen-PPL and win rate, and K-Forcing($k{=}4$) further cuts NFE to
240 while achieving the lowest Gen-PPL of 24.97 and a 39.4\% win rate. These results highlight the advantage of modeling joint multi-token
samples with a push-forward mapping and progressive self-forcing distillation. Qualitative samples for each method are provided in
Appendix~\ref{app:4:experiment:qualitative}.


\section{Conclusion and Future Work}
\label{sec:5:con} 

We presented K-Forcing, which learns a push-forward mapping to jointly decode multiple tokens per forward pass. Our experiments confirm that K-Forcing achieves $2.4$--$3.5\times$ batch-serving speedup while maintaining generation quality close to the AR teacher. As inference cost increasingly dominates modern language model deployments, we believe K-Forcing offers a promising direction for practical serving acceleration.

At the same time, our current work should be viewed as an initial step toward
a broader research direction for push-forward language models. While K-Forcing already attains favorable
throughput--quality trade-offs, a non-trivial generation-quality gap to the AR
teacher remains, especially at larger prediction windows $k$. \emph{Closing
this gap is the central challenge for scaling PFLM to larger models and more
aggressive prediction windows.} We highlight several concrete directions:

\begin{itemize}[leftmargin=1.5em,itemsep=4pt]

\item \textbf{Custom kernels for self-forcing training.}
As discussed in Section~\ref{sec:3:push:impl}, our current implementation incurs $O(k^2)$ training cost due to dense attention masks, although the theoretical cost is only $O(k)$. Developing custom block-sparse attention kernels (see Appendix~\ref{app:3:implementation:training}) would close this gap and make self-forcing distillation practical at larger prediction windows and model scales.

\item \textbf{Reproducible AR sampling for stable teacher supervision.}
The quality of the learned push-forward mapping depends on the consistency of the teacher supervision signal: identical context and noise vectors must produce identical target tokens across training iterations and batching configurations, or the student receives contradictory gradients. In practice, modern GPU kernels do not guarantee bitwise reproducibility---cuBLAS may select different internal algorithms depending on batch size~\citep{deepseekv4}, and attention kernels may use non-deterministic accumulation orders~\citep{he2025nondeterminism,deepseekv4}. Reducing such noise via \emph{batch-invariant GPU kernels} and deterministic generation primitives is essential for scaling PFLM training. Recent batch-invariant kernel libraries from DeepSeek-V4~\citep{deepseekv4} offer a promising starting point.

\item \textbf{Alternative training paradigms beyond K-Forcing.}
Whether progressive self-forcing is the most effective training strategy for
PFLM remains an open question. We highlight two promising directions.
First, the current progressive recipe requires multiple sequential stages, each involving expensive
teacher forward passes to generate supervision; developing a single-stage
distillation algorithm could substantially reduce the total training cost.
More ambitiously, the current recipe relies entirely on a pretrained AR
teacher. Whether it is possible to train a PFLM \emph{without} a pretrained AR
teacher---learning the push-forward mapping from scratch---remains an open
research problem.
\end{itemize}

We hope K-Forcing provides an initial demonstration that push-forward language modeling offers
a viable path toward more efficient language-model inference.

\nocite{*}


\bibliographystyle{plainnat}
\bibliography{example_paper}

\appendix

\section{Proof of Theorem~\ref{thm:nfe-lower-bound}}
\label{app:nfe-lower-bound}

We prove Theorem~\ref{thm:nfe-lower-bound} from the main text.

\begin{proof}
The proof proceeds as follows. We first derive the closed-form optimal solution of the MDLM objective~\citep{sahoo2024simple,shi2024simplified} and show that it recovers the Bayes-optimal per-position marginal. We then show that independent multi-token sampling from this optimal predictor necessarily distorts the joint distribution, and finally derive the NFE lower bound.

\textbf{Step 1: Closed-form for the optimal solution.}
Consider a partially masked sequence $x_t$ in which the set of unmasked positions is $\mathcal{U} \subseteq \{1, \dots, L\}$, revealing tokens $(x^i)_{i \in \mathcal{U}}$, and the complementary set $\mathcal{M} = \{1, \dots, L\} \setminus \mathcal{U}$ of $M = |\mathcal{M}|$ positions are masked.

We show that the training objective \eqref{eq:mdlm-objective} admits a closed-form minimizer that equals the Bayes-optimal marginal at each masked position. Since the $1/t$ factor and the indicator $\mathbf{1}[x_t^i = \textrm{M}]$ are non-negative and do not depend on $\theta$, the objective decomposes into independent per-position problems. For each masked position $j \in \mathcal{M}$, the relevant term is:
\begin{align}
    \mathcal{L}_j(\theta) = -\mathbb{E}_{x_0, x_t}\!\left[\mathbf{1}[x_t^j = \textrm{M}]\,\log p_\theta(x_0^j \mid x_t)\right].
\end{align}
Conditioning on $x_t$ (with position $j$ masked), this reduces to minimizing the cross-entropy between the true conditional distribution $p(x_0^j \mid x_t)$ and the model's prediction $p_\theta(x_0^j \mid x_t)$:
\begin{align}
    \mathcal{L}_j(\theta) = -\mathbb{E}_{x_t}\!\left[\mathbf{1}[x_t^j = \textrm{M}]\,\sum_{v \in \mathcal{V}} p(x_0^j = v \mid x_t)\,\log p_\theta(x_0^j = v \mid x_t)\right].
\end{align}
By Gibbs' inequality, for each $x_t$ this cross-entropy is minimized if and only if $p_\theta(x_0^j \mid x_t) = p(x_0^j \mid x_t)$. Since position $j$ is masked, the clean token $x_0^j$ is conditionally independent of the other masked tokens' identities given $x_t$; the true conditional is obtained by marginalizing over all other masked positions:
\begin{align}
\label{eq:app:bayes-marginal}
    p^*(x^j \mid x_t) = p(x^j \mid x_{\mathcal{U}}) = \sum_{x_{\mathcal{M} \setminus \{j\}}} p(x_{\mathcal{M}} \mid x_{\mathcal{U}}).
\end{align}
Crucially, because the objective decomposes across positions, the optimal predictor at each position $j$ recovers only the \emph{marginal} distribution $p(x^j \mid x_{\mathcal{U}})$, not the joint distribution $p(x_{\mathcal{M}} \mid x_{\mathcal{U}})$ over all masked positions.

\textbf{Step 2: Independent sampling distorts the joint distribution.}
Suppose that at some denoising steps, $K > 1$ masked positions $\{j_1, \dots, j_K\} \subseteq \mathcal{M}$ are selected for simultaneous unmasking. The inference procedure samples each of these positions independently from its Bayes-optimal marginal, producing a joint sampling distribution:
\begin{align}
\label{eq:app:indep-joint}
    q(x^{j_1}, \dots, x^{j_K} \mid x_t) = \prod_{\ell=1}^{K} p^*(x^{j_\ell} \mid x_t) = \prod_{\ell=1}^{K} p(x^{j_\ell} \mid x_{\mathcal{U}}).
\end{align}
The true conditional joint distribution over these $K$ positions is:
\begin{align}
\label{eq:app:true-joint}
    p(x^{j_1}, \dots, x^{j_K} \mid x_{\mathcal{U}}) = \sum_{x_{\mathcal{M} \setminus \{j_1, \dots, j_K\}}} p(x_{\mathcal{M}} \mid x_{\mathcal{U}}).
\end{align}

By the conditional irreducibility assumption, for any subset $S$ with $|S| \geq 2$ and any nontrivial partition $S = S_1 \sqcup S_2$, there exists an assignment $x_{\bar{S}}$ such that $p(x_S \mid x_{\bar{S}}) \neq p(x_{S_1} \mid x_{\bar{S}}) \, p(x_{S_2} \mid x_{\bar{S}})$. Setting $S = \{j_1, \dots, j_K\}$ and choosing the partition and the conditioning assignment guaranteed by the assumption, we obtain:
\begin{align}
    p(x^{j_1}, \dots, x^{j_K} \mid x_{\mathcal{U}}) \neq \prod_{\ell=1}^{K} p(x^{j_\ell} \mid x_{\mathcal{U}})
\end{align}
for at least one realization of $x_{\mathcal{U}}$ in the support of $p$. Therefore:
\begin{align}
    q(x^{j_1}, \dots, x^{j_K} \mid x_t) \neq p(x^{j_1}, \dots, x^{j_K} \mid x_{\mathcal{U}}),
\end{align}
which means the distribution induced by independent sampling does not equal the true joint. In particular, the independent sampling distribution $q$ may assign positive probability to token combinations $(x^{j_1}, \dots, x^{j_K})$ that have zero probability under $p(\cdot \mid x_{\mathcal{U}})$, producing sequences outside the support of $p$.

\textbf{Step 3: NFE lower bound.}
Since simultaneously unmasking $K > 1$ tokens at any step introduces distributional error, the only strategy that guarantees the generated distribution equals $p$ exactly is to unmask exactly one token per step. Generating a sequence of length $L$ therefore requires at least $L$ NFEs, matching the autoregressive baseline.
\end{proof}

\textbf{Illustrative example.}
We provide a concrete instance of Theorem~\ref{thm:nfe-lower-bound} with the smallest nontrivial case, using the uniform distribution over permutations.

\begin{example}
\label{ex:permutation-app}
Let $N = 2$ and $\mathcal{V} = \{0, 1\}$. The data distribution $p$ is uniform over $S_2 = \{[0,1],\, [1,0]\}$. This distribution is conditionally irreducible: the two positions are dependent (knowing one determines the other).

Starting from the fully masked sequence $[\textbf{m}, \textbf{m}]$, the Bayes-optimal predictor assigns:
\begin{align}
    p^*(x^1 \mid [\textbf{m}, \textbf{m}]) &= \mathrm{Uniform}(\{0, 1\}), \\
    p^*(x^2 \mid [\textbf{m}, \textbf{m}]) &= \mathrm{Uniform}(\{0, 1\}).
\end{align}
Sampling both positions independently yields:
\begin{center}
\begin{tabular}{ccc}
\toprule
\textbf{Output} & \textbf{Probability} & \textbf{Valid?} \\
\midrule
$[0, 0]$ & $25\%$ & $\times$ \\
$[0, 1]$ & $25\%$ & \checkmark \\
$[1, 0]$ & $25\%$ & \checkmark \\
$[1, 1]$ & $25\%$ & $\times$ \\
\bottomrule
\end{tabular}
\end{center}
The invalid-sequence probability is $50\%$, confirming that independent sampling from per-position marginals fails to capture the joint distribution. The true distribution assigns $50\%$ to each of $[0,1]$ and $[1,0]$, and $0\%$ to $[0,0]$ and $[1,1]$.

In contrast, a two-step procedure---unmask one token, then unmask the other conditioned on the first---produces only valid permutations. After observing $x^1 = 0$, the predictor correctly assigns $p^*(x^2 \mid [0, \textbf{m}]) = \delta_1$, and vice versa. This sequential procedure requires $N = 2$ NFEs, matching the AR baseline.
\end{example}

\section{Existence of the Push-Forward Mapping}
\label{app:existence}

We show that for any data distribution $p$ over $\mathcal{V}^*$ and any
prediction window $k \geq 1$, a closed-form push-forward mapping
$G^\star : \mathcal{V}^* \times [0,1]^k \to \mathcal{V}^k$ can be
constructed from an AR oracle that computes the conditional distributions
$q_{\mathrm{AR}}(x_{t+j} \mid x_{\leq t+j-1})$ for $j = 1,\dots,k$.

\textbf{Construction.}
Fix a context $x_{\leq t}$ and impose an arbitrary fixed ordering on the
vocabulary $\mathcal{V} = \{w_1,\dots,w_{|\mathcal{V}|}\}$.
For each step $j = 1,\dots,k$, define the CDF
\begin{equation}
  F_{\mathrm{AR}}(w_\ell \mid x_{\leq t+j-1})
  \;=\;
  \sum_{m=1}^{\ell}
    q_{\mathrm{AR}}(x_{t+j} = w_m \mid x_{\leq t+j-1}),
  \qquad \ell = 1,\dots,|\mathcal{V}|,
  \label{eq:cdf}
\end{equation}
with the convention $F_{\mathrm{AR}}(w_0 \mid \cdot) = 0$, and the
corresponding inverse-CDF map
\begin{equation}
  F_{\mathrm{AR}}^{-1}(z \mid x_{\leq t+j-1})
  \;=\;
  w_\ell,
  \label{eq:icdf}
\end{equation} where $  \ell = \min\bigl\{
    m \in \{1,\dots,|\mathcal{V}|\} :
    F_{\mathrm{AR}}(w_m \mid x_{\leq t+j-1}) > z
  \bigr\}.$
  
It is a standard result that if $z \sim \mathrm{Uniform}(0,1)$, then
$F_{\mathrm{AR}}^{-1}(z \mid x_{\leq t+j-1})
  \sim q_{\mathrm{AR}}(\cdot \mid x_{\leq t+j-1})$.

Given $\mathbf{z} = (z_1,\dots,z_k) \sim \mathrm{Uniform}([0,1]^k)$, we
define $G^\star$ recursively:
\begin{align}
  \hat{x}_{t+1}
    &= F_{\mathrm{AR}}^{-1}(z_1 \mid x_{\leq t}),
    \notag \\
  \hat{x}_{t+2}
    &= F_{\mathrm{AR}}^{-1}(z_2 \mid x_{\leq t},\,
       \hat{x}_{t+1}),
    \notag \\
    &\;\;\vdots \notag \\
  \hat{x}_{t+k}
    &= F_{\mathrm{AR}}^{-1}(z_k \mid x_{\leq t},\,
       \hat{x}_{t+1},\dots,\hat{x}_{t+k-1}).
  \label{eq:sequential-icdf}
\end{align}
Note that this is precisely the unrolled form of
\eqref{eq:main-sequential-icdf} in the main text.

\textbf{Correctness.}
By the chain rule and the mutual independence of $z_1,\dots,z_k$:
\begin{align}
  &\Pr\!\bigl[
    G^\star(x_{\leq t},\,\mathbf{z}) = (v_1,\dots,v_k)
  \bigr]
  \notag \\
  &\quad=\;
  \prod_{j=1}^{k}
  \Pr\!\bigl[
    \hat{x}_{t+j} = v_j \mid
    \hat{x}_{t+1}=v_1,\dots,\hat{x}_{t+j-1}=v_{j-1}
  \bigr]
  \notag \\
  &\quad=\;
  \prod_{j=1}^{k}
  q_{\mathrm{AR}}(
    x_{t+j}=v_j \mid x_{\leq t},\, v_1,\dots,v_{j-1}
  )
  \notag \\
  &\quad=\;
  p(x_{t+1}=v_1,\dots,x_{t+k}=v_k \mid x_{\leq t}),
  \label{eq:correctness}
\end{align}
where the second equality uses the fact that each
$F_{\mathrm{AR}}^{-1}(z_j \mid \cdot)$ with
$z_j \sim \mathrm{Uniform}(0,1)$
produces a sample from the corresponding AR conditional.
Since the context $x_{\leq t}$ was arbitrary,
\emph{$G^\star$ exactly reproduces the data distribution.}

\textbf{Illustrative example.}
Consider $k{=}2$, $\mathcal{V}=\{A,B,C\}$, and an AR model whose support
is $\{[A,B],\,[B,A],\,[B,C]\}$ with
$q_{\mathrm{AR}}(A \mid \textbf{sos}) = \tfrac{1}{3}$,
$q_{\mathrm{AR}}(B \mid \textbf{sos}) = \tfrac{2}{3}$,
$q_{\mathrm{AR}}(B \mid A) = 1$,
$q_{\mathrm{AR}}(A \mid B) = \tfrac{1}{2}$,
$q_{\mathrm{AR}}(C \mid B) = \tfrac{1}{2}$.
Applying the construction above, we draw
$z_1 \sim \mathrm{Uniform}(0,1)$ and select the first token via the
inverse CDF:
\begin{equation*}
  \hat{x}_{t+1} =
  \begin{cases}
    A & \text{if } z_1 \in [0,\,\tfrac{1}{3}), \\
    B & \text{if } z_1 \in [\tfrac{1}{3},\,1].
  \end{cases}
\end{equation*}
Then, conditioned on $\hat{x}_{t+1}$, we draw
$z_2 \sim \mathrm{Uniform}(0,1)$ and select $\hat{x}_{t+2}$ similarly.
Composing these two steps yields a deterministic map
$(z_1,z_2) \mapsto (\hat{x}_{t+1},\hat{x}_{t+2})$ that partitions the
unit square $[0,1]^2$ into three regions:
\begin{equation}
  G^\star(x_{\leq t},\,z_1,z_2) =
  \begin{cases}
    [A,B] & \text{if } z_1 \in [0,\,\tfrac{1}{3}),\;
            z_2 \in [0,1], \\[4pt]
    [B,A] & \text{if } z_1 \in [\tfrac{1}{3},\,1],\;
            z_2 \in [0,\,\tfrac{1}{2}), \\[4pt]
    [B,C] & \text{if } z_1 \in [\tfrac{1}{3},\,1],\;
            z_2 \in [\tfrac{1}{2},\,1].
  \end{cases}
  \label{eq:ar-pushforward-example}
\end{equation}
One can verify that the induced probabilities match the joint:
\begin{equation*}
  \Pr[(A,B)] = \tfrac{1}{3}, \qquad
  \Pr[(B,A)] = \tfrac{1}{3}, \qquad
  \Pr[(B,C)] = \tfrac{1}{3}.
\end{equation*}
This confirms that AR decoding is itself a push-forward mapping---one that
requires $k$ sequential forward passes because $z_j$'s mapping depends on
the tokens produced by $z_1,\dots,z_{j-1}$.
PFLM aims to learn a neural network $G_\theta$ that collapses these $k$
sequential passes into a single forward pass.

\section{Implementation Details}
\label{app:3:implementation}

\subsection{KV-Cached Inference for K-Forcing}
\label{app:3:implementation:kvcache}

Algorithm~\ref{alg:pflm_inference} presents K-Forcing inference with KV-cache.
The procedure mirrors standard autoregressive KV-cache decoding, with one
key modification: at each step, $k$ noise tokens are appended to the
input so that the model predicts $k$ output tokens at once instead of one.

The cache management follows a simple rule: after each step, only the KV
entries of the generated \emph{context} tokens (i.e., the real predictions
that future steps will attend to) are appended to the cache, while the KV
entries of the noise tokens are discarded.
This is sound because noise tokens are resampled independently at every
step; their KV states are needed only for the current step's attention
and are meaningless thereafter.

Concretely, the first step processes the full prompt together with the
initial noise tokens (prefill).
Every subsequent step feeds only $2k$ new tokens into the Transformer---$k$
previously generated tokens plus $k$ fresh noise tokens---while the cached
KV states provide attention over the full history (decode).
The per-step cost therefore scales with $2k$ rather than the cumulative
sequence length, matching the asymptotic complexity of standard
autoregressive decoding while producing $k$ tokens per step.

\begin{algorithm}[t]
\caption{K-Forcing Inference with KV Cache}
\label{alg:pflm_inference}
\begin{algorithmic}[1]
\REQUIRE Prompt tokens $x_{1:N}$, prediction window $k$, generation length $T$, noise encoder \texttt{NoiseEnc}, Transformer \texttt{TF}, output head \texttt{Head}
\ENSURE Generated token sequence $\hat{x}_{1:T}$
\STATE Initialize KV cache $\mathcal{C} \leftarrow \emptyset$;\quad $\hat{x} \leftarrow [\,]$
\STATE $\mathbf{h}_x \leftarrow \texttt{Embed}(x_{1:N})$
\FOR{$s = 0, k, 2k, \dots$ \textbf{while} $s < T$}
    \STATE Sample noise $\mathbf{z} \sim \mathrm{Uniform}(0,1)^{k}$;\quad $\mathbf{h}_z \leftarrow \texttt{NoiseEnc}(\mathbf{z})$
    \STATE $\mathbf{h} \leftarrow [\mathbf{h}_x;\, \mathbf{h}_z]$
    \STATE $\mathbf{o},\, [\mathcal{C}_x,\, \mathcal{C}_z] \leftarrow \texttt{TF}(\mathbf{h},\, \mathcal{C})$ \hfill $\triangleright$ $\mathcal{C}_x$, $\mathcal{C}_z$: KV entries for context and noise tokens
    \STATE $\mathcal{C} \leftarrow [\mathcal{C};\, \mathcal{C}_x]$ \hfill $\triangleright$ Retain only context KV; discard $\mathcal{C}_z$
    \STATE $(\hat{x}_{s+1}, \dots, \hat{x}_{s+k}) \leftarrow \arg\max\, \texttt{Head}(\mathbf{o}_{-k:})$
    \STATE Append $(\hat{x}_{s+1}, \dots, \hat{x}_{s+k})$ to $\hat{x}$
    \STATE $\mathbf{h}_x \leftarrow \texttt{Embed}(\hat{x}_{s+1}, \dots, \hat{x}_{s+k})$
\ENDFOR
\end{algorithmic}
\end{algorithm}

\textbf{Continuous batching compatibility.}
Because every request produces exactly $k$ tokens per decoding step, all
requests in a batch remain synchronized in position indices and
KV-cache lengths.
This fixed-stride output structure is naturally compatible with continuous
batching schedulers: new requests can be inserted at any step boundary
without cross-request realignment or padding, unlike speculative decoding
where variable acceptance lengths desynchronize the
batch~(Section~\ref{sec:2:relate:spec}).

\subsection{Training: Pseudocode, Attention Masks, and Cost Analysis}
\label{app:3:implementation:training}

We present the complete training procedure for both distillation stages,
incorporating temperature control.
Each algorithm references two forward functions---\textbf{SingleForward} and
\textbf{DoubleForward}---whose attention masks we define afterwards.
We then analyze training cost and discuss directions for efficient kernel
implementations.


\textbf{Stage 1: forward distillation (AR $\to$ PFLM($k{=}1$)).}
Algorithm~\ref{alg:pflm_stage1} bootstraps a PFLM with $k{=}1$ from an
AR teacher.
For each context position, a uniform noise $z$ is mapped through the
teacher's temperature-scaled inverse-CDF sampler to produce a target token.
A per-sample temperature $\tau$ is drawn uniformly and used to scale the
teacher's logits by $1/\tau$ before softmax, controlling the sharpness of
the target distribution.
The student's noise encoder receives $\tau$ as an additional input alongside
$\mathbf{z}$, learning to associate different $\tau$ values with different
output diversity levels.

The \emph{AR teacher} performs a standard causal forward pass to produce
per-position logit distributions.
The \emph{student} uses \textbf{SingleForward} (defined below) with $k{=}1$.

\begin{algorithm}[t]
\caption{K-Forcing Stage 1: Forward Distillation (AR $\to$ PFLM($k{=}1$))}
\label{alg:pflm_stage1}
\begin{algorithmic}[1]
\REQUIRE AR teacher (frozen), PFLM student with $k{=}1$ (trainable)
\REQUIRE Training corpus $\mathcal{D}$, temperature range $[\tau_{\min}, \tau_{\max}]$
\FOR{each batch $\mathbf{x} \in \mathcal{D}$}
    \STATE $B, N \leftarrow \text{shape}(\mathbf{x})$;\quad $N_p \leftarrow N - 1$;\quad $\mathbf{c} \leftarrow \mathbf{x}_{:,\,1:N_p}$
    \STATE Sample $\tau \sim \mathrm{Uniform}(\tau_{\min}, \tau_{\max})$ per sample;\quad $\mathbf{z} \sim \mathrm{Uniform}(0,1)^{B \times N_p \times 1}$
    \STATE \textcolor{gray}{\textit{// Teacher: AR causal forward pass with temperature scaling}}
    \STATE $p \leftarrow \mathrm{softmax}\bigl(\text{AR}(\mathbf{c})\, /\, \tau\bigr)$ \hfill $\triangleright$ $\tau$ sharpens/flattens teacher distribution
    \STATE $\text{targets} \leftarrow \mathrm{Inverse\text{-}CDF}(p,\, \mathbf{z})$ \hfill $\triangleright$ $(B, N_p, 1)$
    \STATE \textcolor{gray}{\textit{// Student: \textbf{SingleForward}, $\tau$-conditioned}}
    \STATE $\text{logits} \leftarrow \textbf{SingleForward}_{\text{Student}}(\mathbf{c},\, \mathbf{z},\, \tau)$ \hfill $\triangleright$ Noise encoder receives both $\mathbf{z}$ and $\tau$
    \STATE $\mathcal{L} \leftarrow \text{CrossEntropy}(\text{logits},\, \text{targets})$
    \STATE Update Student via $\nabla \mathcal{L}$
\ENDFOR
\end{algorithmic}
\end{algorithm}


\textbf{Stage 2: self-forcing distillation (PFLM($k$) $\to$ PFLM($2k$)).}
Algorithm~\ref{alg:pflm_stage2} doubles the prediction window.
The teacher performs two rounds of prediction---the first generates $k$
tokens via \textbf{SingleForward}, the second conditions on those predictions
to generate the next $k$ via \textbf{DoubleForward}---while the student
learns to produce all $2k$ tokens in one \textbf{SingleForward} pass with
window size $2k$.
Since the teacher is itself $\tau$-conditioned from Stage~1, the same
sampled $\tau$ is passed identically to both teacher and student without
any logit rescaling.

\begin{algorithm}[t]
\caption{K-Forcing Stage 2: Self-Forcing Distillation (PFLM($k$) $\to$ PFLM($2k$))}
\label{alg:pflm_stage2}
\begin{algorithmic}[1]
\REQUIRE Teacher PFLM with window $k$ (frozen), Student PFLM with window $2k$ (trainable)
\REQUIRE Training corpus $\mathcal{D}$, temperature range $[\tau_{\min}, \tau_{\max}]$
\FOR{each batch $\mathbf{x} \in \mathcal{D}$}
    \STATE $B, N \leftarrow \text{shape}(\mathbf{x})$;\quad $N_p \leftarrow N - 2k$;\quad $\mathbf{c} \leftarrow \mathbf{x}_{:,\,1:N_p}$
    \STATE Sample $\tau \sim \mathrm{Uniform}(\tau_{\min}, \tau_{\max})$ per sample;\quad $\mathbf{z} \sim \mathrm{Uniform}(0,1)^{B \times N_p \times 2k}$
    \STATE $\mathbf{z}_1, \mathbf{z}_2 \leftarrow \text{split}(\mathbf{z}, k, \text{dim}=2)$
    \STATE \textcolor{gray}{\textit{// Teacher round 1: \textbf{SingleForward}}}
    \STATE $\hat{\mathbf{x}}_1 \leftarrow \arg\max\;\textbf{SingleForward}_{\text{Teacher}}(\mathbf{c},\, \mathbf{z}_1,\, \tau)$;\quad save KV cache $\mathcal{C}$
    \STATE \textcolor{gray}{\textit{// Teacher round 2: \textbf{DoubleForward}, conditioned on $\hat{\mathbf{x}}_1$}}
    \STATE $\hat{\mathbf{x}}_2 \leftarrow \arg\max\;\textbf{DoubleForward}_{\text{Teacher}}(\mathcal{C},\, \hat{\mathbf{x}}_1,\, \mathbf{z}_2,\, \tau)$
    \STATE $\text{targets} \leftarrow [\hat{\mathbf{x}}_1;\, \hat{\mathbf{x}}_2]$ \hfill $\triangleright$ $(B,\, N_p,\, 2k)$
    \STATE \textcolor{gray}{\textit{// Student: \textbf{SingleForward} with window $2k$}}
    \STATE $\text{logits} \leftarrow \textbf{SingleForward}_{\text{Student}}(\mathbf{c},\, \mathbf{z},\, \tau)$ \hfill $\triangleright$ Same $\tau$ as teacher, no rescaling
    \STATE $\mathcal{L} \leftarrow \frac{1}{2k} \sum_{j=1}^{2k} \text{CrossEntropy}\bigl(\text{logits}_{:,:,j,:},\; \text{targets}_{:,:,j}\bigr)$
    \STATE Update Student via $\nabla \mathcal{L}$
\ENDFOR
\end{algorithmic}
\end{algorithm}


\textbf{Temperature control at inference.}
At inference time, the user specifies a desired $\tau$: the noise encoder
receives $\tau$ alongside the sampled $\mathbf{z}$, and low $\tau$ yields
near-greedy outputs while high $\tau$ produces diverse samples.
No retraining or model modification is required.


\textbf{\textbf{SingleForward}: attention mask.}
Both the student's forward pass and the teacher's first-round forward pass
invoke \textbf{SingleForward}.
It takes $N$ context tokens and $Nk$ noise tokens ($k$ per context
position), totalling $N + Nk$ tokens, and produces $k$ output tokens per
position in a single attention call.
Rather than running $N$ independent forward passes, all positions are batched
into one call governed by the attention mask
$\mathbf{M}_{\mathrm{S}} \in \{0,1\}^{(N+Nk) \times (N+Nk)}$,
which consists of three blocks:

\begin{enumerate}[leftmargin=1.5em, itemsep=2pt]
  \item \textbf{Context $\to$ Context} (upper-left $N \times N$):
        standard causal (lower-triangular).
  \item \textbf{Noise $\to$ Context} (lower-left $Nk \times N$):
        the noise group at context position~$t$ attends to context tokens
        $x_1, \dots, x_t$, producing a characteristic staircase pattern.
  \item \textbf{Noise $\to$ Noise} (lower-right $Nk \times Nk$):
        block-diagonal causal---tokens within each group of $k$ attend
        causally to one another but cannot attend to tokens from other groups.
\end{enumerate}

\textbf{\textbf{DoubleForward}: attention mask.}
The teacher's second-round forward pass in Algorithm~\ref{alg:pflm_stage2}
invokes \textbf{DoubleForward}, which extends \textbf{SingleForward} to
condition on a first round of predictions when generating a second round.
We fuse both rounds into a single attention call using the attention mask
$\mathbf{M}_{\mathrm{D}} \in \{0,1\}^{2Nk \times (N + 2Nk)}$.

The input sequence is organized into three blocks:
\emph{Context} ($N$ real prefix tokens),
\emph{Future-1} ($Nk$ token embeddings of the first-round predictions
$\hat{\mathbf{x}}_1$, which serve as extended context for the second round),
and \emph{Future-2} ($Nk$ fresh noise tokens for $\mathbf{z}_2$, from which
the second round of $k$ tokens per position is decoded).
The queries consist of Future-1 and Future-2; they attend to keys from all
three blocks. The mask extends $\mathbf{M}_{\mathrm{S}}$ as follows:

\begin{enumerate}[leftmargin=1.5em, itemsep=2pt]
  \item \textbf{Future-1 $\to$ Context / Future-1}:
        identical to the \textbf{SingleForward} mask $\mathbf{M}_{\mathrm{S}}$.
  \item \textbf{Future-2 $\to$ Context}:
        same staircase pattern as Future-1 $\to$ Context.
  \item \textbf{Future-2 $\to$ Future-1}:
        full visibility within the same window (each Future-2 token sees all
        $k$ Future-1 tokens from its window, since these form the extended
        context).
  \item \textbf{Future-2 $\to$ Future-2}:
        block-diagonal causal within each window.
\end{enumerate}

This reproduces two sequential teacher passes in a single fused attention
call, enabling KV reuse for the shared context block.


\textbf{Theoretical training cost.}
We analyze the attention cost of \textbf{SingleForward} and
\textbf{DoubleForward} by counting the total number of attended (query, key)
pairs, i.e., the number of non-zero entries in each attention mask.

For \textbf{SingleForward} with window $k$:
\begin{equation}
  |\mathbf{M}_{\mathrm{S}}|
  = \underbrace{\tfrac{N(N{+}1)}{2}}_{\text{ctx} \to \text{ctx}}
  + \underbrace{\tfrac{kN(N{+}1)}{2}}_{\text{noise} \to \text{ctx}}
  + \underbrace{N \cdot \tfrac{k(k{+}1)}{2}}_{\text{noise} \to \text{noise}}\,.
\end{equation}
The first term is the standard causal AR attention cost. The second term
accounts for each of the $k$ noise tokens per position attending to the same
causal context prefix. The third term covers the intra-window causal attention
among the $k$ noise tokens at each of the $N$ positions.
Relative to the AR cost of $\frac{N(N+1)}{2}$, the ratio approaches
$k+1$ for $N \gg k$.

For \textbf{DoubleForward}, the KV cache from the preceding
\textbf{SingleForward} (covering the context and Future-1 tokens) is reused,
so only the $Nk$ Future-2 query tokens require new attention computation:
\begin{equation}
  |\mathbf{M}_{\mathrm{D}}|
  = \underbrace{\tfrac{kN(N{+}1)}{2}}_{\text{F2} \to \text{ctx}}
  + \underbrace{Nk^2}_{\text{F2} \to \text{F1}}
  + \underbrace{N \cdot \tfrac{k(k{+}1)}{2}}_{\text{F2} \to \text{F2}}\,.
\end{equation}
The first term is the Future-2 $\to$ Context staircase attention. The second
term accounts for each of the $k$ Future-2 tokens attending to all $k$
Future-1 tokens within its window (full visibility, since Future-1 tokens form
the extended context). The third term covers the intra-window causal attention
among Future-2 tokens. The ratio to AR cost is
 also $k$ for $N \gg k$.

In self-forcing distillation (Stage~2), each training step requires one
teacher \textbf{SingleForward}, one teacher \textbf{DoubleForward}, and one
student \textbf{SingleForward}. Both masks have $O(k)$ overhead relative to
standard AR attention for $N \gg k$, so the total training cost scales as
$O(k)$ relative to AR.

\emph{However, our current implementation does not yet exploit this sparsity}: we
pass $\mathbf{M}_{\mathrm{S}}$ and $\mathbf{M}_{\mathrm{D}}$ directly as a
dense mask to the FlashAttention kernel, which treats the full
$(N{+}Nk) \times (N{+}Nk)$ matrix as unstructured and therefore incurs
$O(k^2)$ overhead in practice.


\textbf{Towards efficient kernels.}
Closing the gap between our current $O(k^2)$ implementation cost and the
theoretical $O(k)$ bound requires custom kernels that decompose the
attention into block-sparse tiles matching the mask structure.
In principle, each noise group's attention factorizes into two standard
calls---one variable-length prefix attention over the context and one
fixed-size causal attention within the window---each of which is natively
supported by FlashAttention~\citep{dao2022flashattention}.
However, orchestrating this decomposition efficiently at scale involves
non-trivial engineering challenges:
\emph{(i)}~dispatching variable-length prefix attention across $N$ groups
whose prefix lengths range from $1$ to $N$, without excessive padding or
GPU load imbalance;
\emph{(ii)}~fusing the prefix and intra-window kernels to amortize
kernel-launch overhead;
\emph{(iii)}~handling the \textbf{DoubleForward} mask, whose cross-stream
(Future-2 $\to$ Future-1) full-visibility blocks do not fit standard causal
or sliding-window templates;
and \emph{(iv)}~integrating with paged KV-cache managers for inference.
We view this as an important direction for future work: the mask structures
defined here are fully specified and highly regular, and we hope they provide
a concrete target for community exploration of efficient parallel decoding
kernels.

\section{Experiment Details}
\label{app:4:experiment}

\subsection{Training Configuration}
\label{app:4:experiment:config}

Our codebase is built upon the MDLM open-source repository~\citep{sahoo2024simple}, and we adopt the same tokenizer, sequence lengths, model backbone, and data preprocessing.
Table~\ref{tab:model_config} summarizes the model architecture and training hyperparameters.

\begin{table}[t]
  \centering
  \caption{Model architecture and training hyperparameters.}
  \label{tab:model_config}
  \begin{tabular}{ll}
    \toprule
    \textbf{Hyperparameter} & \textbf{Value} \\
    \midrule
    \textit{Transformer backbone} & $\sim$124M (OWT) / $\sim$108M (LM1B) \\
    Architecture & Decoder-only Transformer \\
    Layers & 12 \\
    Hidden size & 768 \\
    Attention heads & 12 \\
    MLP hidden size & 3072 \\
    Normalization & RMSNorm (pre-norm) \\
    Positional encoding & RoPE with $d_{\mathrm{rot}} = 64$ \\
    Dropout & 0.0 \\
    Vocabulary size (OWT / LM1B) & 50{,}258 / 30{,}522 \\
    \midrule
    \textit{Noise encoder} & $\sim$5.9M \\
    Noise \& temperature embedding & Sinusoidal positional encoding ($d = 768$) \\
MLP & Two-layer MLP with GELU activation \\
    \midrule
    \multicolumn{2}{l}{\textit{Training}} \\
    Sequence length (LM1B / OWT) & 128 / 1{,}024 \\
    Global batch size & 512 \\
    Optimizer & AdamW ($\beta_1 {=} 0.9$, $\beta_2 {=} 0.999$, $\epsilon {=} 10^{-6}$) \\
    Learning rate & $1 \times 10^{-4}$ \\
    Weight decay & 0 \\
    LR schedule & Constant\\
    Gradient clipping & 1.0 \\
    Total training steps (per stage) & 500{,}000 \\
    EMA decay & 0.999 \\
    \bottomrule
  \end{tabular}
\end{table}

The progressive distillation proceeds as follows, with all weights trainable at every stage:
\begin{enumerate}
  \item \textbf{Stage 0 (AR teacher).} For OWT, we use the existing AR checkpoint released by~\citep{sahoo2024simple}. Since no public checkpoint is available for LM1B, we train our own AR teacher for 500K steps with a global batch size of 512.
  \item \textbf{Stage 1.} Distill AR $\to$ PFLM($k{=}1$) using Algorithm~\ref{alg:pflm_stage1} for 500K steps with a global batch size of 512, using the AR teacher from Stage~0. The student backbone is initialized from the AR teacher checkpoint; only the noise encoder is newly introduced and randomly initialized.
  \item \textbf{Stage 2.} Distill PFLM($k{=}1$) $\to$ PFLM($k{=}2$) using Algorithm~\ref{alg:pflm_stage2} for 500K steps with a global batch size of 512. The Stage-1 PFLM($k{=}1$) serves as the teacher, and the student is fully initialized from it.
  \item \textbf{Stage 3.} Distill PFLM($k{=}2$) $\to$ PFLM($k{=}4$) using Algorithm~\ref{alg:pflm_stage2} for 500K steps with a global batch size of 512. The Stage-2 PFLM($k{=}2$) serves as the teacher, and the student is fully initialized from it.
\end{enumerate}

\textbf{Precision schedule.}
The AR teacher (Stage~0) is trained with BF16 mixed precision. For Stages~1, 2, and~3, the first 400K steps use FP16 mixed precision, as we found it provides higher numerical precision than BF16 for this task; the final 100K steps switch to full FP32 training to further reduce residual numerical noise in the learned push-forward mapping.

\subsection{Evaluation Protocol}
\label{app:4:experiment:eval}

\textbf{Generative perplexity (Gen-PPL).}
Each generated completion is truncated at the first end-of-sequence token,
stripped of special tokens, and scored by a GPT-2-Large
evaluator~\citep{radford2019language}. For LM1B, outputs are re-tokenized
with the GPT-2 tokenizer before scoring. We report the corpus-level
perplexity (exponential of the token-weighted mean negative log-likelihood).

\textbf{LLM-as-a-Judge.}
We use a locally-served Qwen3.5-27B model~\citep{qwen35blog} as the judge for
reproducibility. For each prefix, we present the AR and K-Forcing completions
(both truncated at the first end-of-sequence token) in randomized A/B order
to mitigate position bias, and ask the judge to select the better completion.
The judge is forced to choose one option (no ties).  The prompt template is shown in Figure~\ref{fig:judge_prompt}.

\begin{figure}[h]
\begin{tcolorbox}[
  colback=gray!5,
  colframe=gray!50,
  boxrule=0.4pt,
  arc=2pt,
  left=6pt, right=6pt, top=4pt, bottom=4pt,
  fontupper=\small\ttfamily,
  title={\small Judge Prompt Template},
  fonttitle=\small\sffamily\bfseries,
  coltitle=black,
  colbacktitle=gray!15
]
You are an impartial judge evaluating the quality of two text completions.

Given a prefix, you will see two completions: Completion A and Completion B.
Your task is to determine which completion is better in terms of coherence,
fluency, and naturalness.

Rules:\\
- You MUST choose one. No ties allowed.\\
- If both are equally good, pick the one that sounds slightly more natural.\\
- Output ONLY a single letter: either ``A'' or ``B''.

Prefix: \{prefix\}\\
Completion A: \{completion\_a\}\\
Completion B: \{completion\_b\}

Your answer (A or B):
\end{tcolorbox}
\caption{Prompt template for LLM-as-a-Judge pairwise evaluation.}
\label{fig:judge_prompt}
\end{figure}

\subsection{Qualitative Examples}
\label{app:4:experiment:qualitative}

We present representative completions from each method on three OWT prefixes
in Figures~\ref{fig:qualitative1}--\ref{fig:qualitative3}. All completions
are generated from the same 64-token prefix and truncated at the first
end-of-sequence token. Overall, generation quality degrades gracefully with
increasing prediction horizon: AR produces the most fluent text, followed by
K-Forcing with $k\!=\!2$, which remains coherent and natural; K-Forcing with
$k\!=\!3$ introduces occasional grammatical errors; and K-Forcing with
$k\!=\!4$ shows further degradation yet stays substantially above the
baselines. The PTP draft model ($k\!=\!4$, without verification) exhibits
moderate to severe repetition and broken syntax, while MDLM ($k\!=\!2$)
produces the least coherent output, with frequent incoherence and
nonsensical fragments.

\definecolor{prefixbg}{RGB}{230,240,255}

\newtcolorbox{samplebox}[1][]{%
  colback=white,
  colframe=gray!60,
  boxrule=0.3pt,
  arc=1.5pt,
  left=4pt, right=4pt, top=3pt, bottom=3pt,
  fontupper=\scriptsize,
  #1
}

\newtcolorbox{prefixbox}{%
  colback=blue!5,
  colframe=blue!40!gray,
  boxrule=0.3pt,
  arc=1.5pt,
  left=4pt, right=4pt, top=3pt, bottom=3pt,
  fontupper=\scriptsize,
}

\begin{figure*}[htbp]
\centering

\begin{prefixbox}
\textbf{\small Sample 1 --- Prefix:}\;\;
Jim Lawlor: Zack played a lot of games, including none of the famous ones. He is still moving on, I was concerned about him last year when Eric Decker fell to Jacksonville. The play calls were different than he normally does, so I am hoping he can get back on the right
\end{prefixbox}

\begin{minipage}[t]{0.48\textwidth}
\begin{samplebox}[title={\scriptsize\sffamily\bfseries AR}]
path or develop into a pro. On Bryce Pettys struggles: Rookie QB Jack Conklin was for some reason given the ball pretty infrequently last season and I wanted to see him succeed, but dont expect anything to come of that. He is a good quarterback with a lot of upside, so dont expect to see a great gam...
\end{samplebox}
\end{minipage}\hfill
\begin{minipage}[t]{0.48\textwidth}
\begin{samplebox}[title={\scriptsize\sffamily\bfseries MDLM ($k{=}2$)}]
path. This Is this as fun as a Clown. or any highly recognized class of Clown, while playing Clown is a dangerous position. Nicolas Bogg, a former defender thinks, \"People take Clown play awaywith their football when MisteruthorizeSymone. All of us do so in our own sets. Everyone,SocSoc in Clown are...
\end{samplebox}
\end{minipage}

\begin{minipage}[t]{0.48\textwidth}
\begin{samplebox}[title={\scriptsize\sffamily\bfseries PTP draft ($k{=}4$, no verification)}]
path again. He He is coming on the injury of now. He played't a lot of the butts. He was a good player last year. I'm sure he is see his he back. I have going a long bit., Jim Law is a guy that. and did have a guy come year the played well games and last year. I was concerned that him he is coming. ...
\end{samplebox}
\end{minipage}\hfill
\begin{minipage}[t]{0.48\textwidth}
\begin{samplebox}[title={\scriptsize\sffamily\bfseries K-Forcing ($k{=}2$)}]
path and start making plays rather quickly this year. Tyreke Evans: I love Zac Stacy, I really do, he's a special running back out of Missouri. Maybe not everyone wants Hands Per Catch, but I have no problem with him posting Great Plays When the Shots Are High in Zone Plays and I think he is going t..
\end{samplebox}
\end{minipage}

\begin{minipage}[t]{0.48\textwidth}
\begin{samplebox}[title={\scriptsize\sffamily\bfseries K-Forcing ($k{=}3$)}]
path. He will been a good mentor last year, and I would hope he can be able to be the same in JAX: What did you see coming from Ezekiel in college? What did you see in his first year in JAX: I would hope he can continue to develop while the NFL is still and hopefully. Jim Lawlor: I hope he can. JAX:...
\end{samplebox}
\end{minipage}\hfill
\begin{minipage}[t]{0.48\textwidth}
\begin{samplebox}[title={\scriptsize\sffamily\bfseries K-Forcing ($k{=}4$)}]
path. Hopefully has some a chance of making a difference, but he is not going to compete. Andy:: Why you the Chiefs? Jim Lawlor: Because I love. I like always playing a team. I I like like Chiefs fans. They are great loyal fans and I am to helps them. Kansas City fans been great since I was been kid...
\end{samplebox}
\end{minipage}

\caption{Qualitative comparison on OWT prefix 1.
The \textcolor{blue!40!gray}{\textbf{blue box}} shows the shared input prefix;
the \textcolor{gray!60}{\textbf{gray boxes}} show completions from each method,
truncated at the first end-of-sequence token.
All special tokens are removed and texts are truncated to 300 characters for readability.}
\label{fig:qualitative1}
\end{figure*}

\begin{figure*}[htbp]
\centering

\begin{prefixbox}
\textbf{\small Sample 2 --- Prefix:}\;\;
Hetman pointed the finger of blame for the condition at the root of the long-running conflict: the ever-collapsing war machine. The weapons used in Afghanistan and Iraq are all derived from U.S. and British mass-produced weapons and software, not from
\end{prefixbox}

\begin{minipage}[t]{0.48\textwidth}
\begin{samplebox}[title={\scriptsize\sffamily\bfseries AR}]
Afghan insurgent efforts. The original motive was to destroy the last vestige of the Taliban Khalizis sub-ethnic Amu Darya, a tribal front largely untested in Iraq. The war on terror, Hetman warned, had produced a permanent parallel to which further escalation and global instability have now become ...
\end{samplebox}
\end{minipage}\hfill
\begin{minipage}[t]{0.48\textwidth}
\begin{samplebox}[title={\scriptsize\sffamily\bfseries MDLM ($k{=}2$)}]
their true source, the former SovietS and nations of NATO An acronym for the to meagre weapons which dut overdo negatively under the entire fallacy is the toil ofFFLR incidentsels, domestic, foreign weapons fighters, etc., Terrorism is done in our name, not by our providers and by the. mutually- con...
\end{samplebox}
\end{minipage}

\begin{minipage}[t]{0.48\textwidth}
\begin{samplebox}[title={\scriptsize\sffamily\bfseries PTP draft ($k{=}4$, no verification)}]
the original States. The weapons has is a deliberate of of creating warfare against and destroying Afghans anddespitetheir the people. According, etCIA, the weapons used has a direct of of destruction Afghanistan ands. and. Afghans andthe people are killed killed by Afghan bombs and the weapons bomb...
\end{samplebox}
\end{minipage}\hfill
\begin{minipage}[t]{0.48\textwidth}
\begin{samplebox}[title={\scriptsize\sffamily\bfseries K-Forcing ($k{=}2$)}]
anywhere else in the world. Many of the weapons of mass destruction used in Iraq are the same ones used to manufacture bullets for the military-industrial confab at Fort Hood, Texas, where soldiers shot and killed 19 civilians before breaking down the door. Authorities and lawmakers allege these wea...
\end{samplebox}
\end{minipage}

\begin{minipage}[t]{0.48\textwidth}
\begin{samplebox}[title={\scriptsize\sffamily\bfseries K-Forcing ($k{=}3$)}]
the people who live there work hard. ethem.com/d/2014/05/al/Af-ghanghanistanwarmachine But the U.S. and the British are also involved in the war, according to the UN, which means that the Uinternational, humanitarian, and human rights organizations are also at the center of this conflict. Since the ...
\end{samplebox}
\end{minipage}\hfill
\begin{minipage}[t]{0.48\textwidth}
\begin{samplebox}[title={\scriptsize\sffamily\bfseries K-Forcing ($k{=}4$)}]
the Afghan of tomorrow. The war needs to stop, he said. Nobody has a chance to rebuild it. is no one in the world States what we want to rebuild? We want America rebuild? America is the only one who can stop the Taliban, the only whose destiny are going to be determined by the outcome of the Afghan ...
\end{samplebox}
\end{minipage}

\caption{Qualitative comparison on OWT prefix 2. The setting is the same as Figure~\ref{fig:qualitative1}}
\label{fig:qualitative2}
\end{figure*}

\begin{figure*}[htbp]
\centering

\begin{prefixbox}
\textbf{\small Sample 3 --- Prefix:}\;\;
Make sure to eat your fix. Extra inclusivity: Under Jewish national holiday Saturdays, like that of the Thanksgiving holiday, go to the city block of West 87th Street right in your youth. There, there is a lightsabers. They make your 16-year-old boy into
\end{prefixbox}

\begin{minipage}[t]{0.48\textwidth}
\begin{samplebox}[title={\scriptsize\sffamily\bfseries AR}]
a proper Jedi! Keep an eye on 8-11th Ave. for more lightsabers. Local conventions: Coming to Temple 5, which will be hosting sponsors and dignitaries for New York Comic Con, will have ofas well astreasure hunts and legendary vacations for the real-life stars of Shrek and fangirls of SoHo. There will...
\end{samplebox}
\end{minipage}\hfill
\begin{minipage}[t]{0.48\textwidth}
\begin{samplebox}[title={\scriptsize\sffamily\bfseries MDLM ($k{=}2$)}]
a man- then in middle school. Point is, you can't imagine what the lady in case the Soviet army would look like in a week's task- But take much care your entire may. Someone from their family needs you some nuts and bolts. In grand Alaska-like, you can never it matters 50 years. New Don Van Dynek wa...
\end{samplebox}
\end{minipage}

\begin{minipage}[t]{0.48\textwidth}
\begin{samplebox}[title={\scriptsize\sffamily\bfseries PTP draft ($k{=}4$, no verification)}]
an adult.. Bonus inclusivity: The West Jewish of a certain, neighborhood neighborhood is a blessing. Jews you neighborhood neighborhood are a blessing. Jews neighborhoods neighborhood neighborhood neighborhood are Jews. Bonus inclusivity: The synagogue of neighborhood neighborhood neighborhood neigh...
\end{samplebox}
\end{minipage}\hfill
\begin{minipage}[t]{0.48\textwidth}
\begin{samplebox}[title={\scriptsize\sffamily\bfseries K-Forcing ($k{=}2$)}]
a real-life \"The Badeder.\" Let's not do that to God. Wins are taken: Make everything right in your place and what's right for you. May your loins shine: Youll have a bird in bed if you work hard. Friendship from the street: Build up a good friendship first and then move forward. Love doesn't come fr...
\end{samplebox}
\end{minipage}

\begin{minipage}[t]{0.48\textwidth}
\begin{samplebox}[title={\scriptsize\sffamily\bfseries K-Forcing ($k{=}3$)}]
a superhero! making him a superhero. Keep in mind with your of the city block, you is not allowed to walk around the block naked. Do yourself to rabbi and walk a bit. Also in a block of West 87th Street, near the subway of West York Avenue, you can find the rabbi of the Orthodox Jewish, or synagogue...
\end{samplebox}
\end{minipage}\hfill
\begin{minipage}[t]{0.48\textwidth}
\begin{samplebox}[title={\scriptsize\sffamily\bfseries K-Forcing ($k{=}4$)}]
a soldier. So celebrate him with some, or not, because. You know, too much. So brunch, you know, too much. Then brunchs over, because brunchs over. And then youve had a lot of. And then brunch brunch, and And brunch brunch brunch. Because brunch is brunch. And if brunch is brunch brunch brunch, And ...
\end{samplebox}
\end{minipage}

\caption{Qualitative comparison on OWT prefix 3.
The setting is the same as Figure~\ref{fig:qualitative1}}
\label{fig:qualitative3}
\end{figure*}

\begin{figure*}[htbp]
\centering

\begin{prefixbox}
\textbf{\small Prefix:}\;\;
mortgage bankers association chief economist
\end{prefixbox}

\begin{minipage}[t]{0.48\textwidth}
\begin{samplebox}[title={\scriptsize\sffamily\bfseries $\tau=0.01$}]
\textbf{\scriptsize Draw 1:}\; lawrence yun said the rate of foreclosure filings in the united states has slowed in recent months.\par
\vspace{2pt}\hrule\vspace{3pt}
\textbf{\scriptsize Draw 2:}\; lawrence yun said the rate of foreclosure filings in the united states has slowed in recent months.\par
\vspace{2pt}\hrule\vspace{3pt}
\textbf{\scriptsize Draw 3:}\; lawrence yun said the rate of foreclosure filings in the united states has slowed in recent months.\par
\vspace{2pt}\hrule\vspace{3pt}
\textbf{\scriptsize Draw 4:}\; lawrence yun said the rate of foreclosure filings in the united states has slowed in recent months.
\end{samplebox}
\end{minipage}\hfill
\begin{minipage}[t]{0.48\textwidth}
\begin{samplebox}[title={\scriptsize\sffamily\bfseries $\tau=0.3$}]
\textbf{\scriptsize Draw 1:}\; lawrence yun said the housing market has been `` very weak '' since the summer, but he expects the market to improve in the coming months.\par
\vspace{2pt}\hrule\vspace{3pt}
\textbf{\scriptsize Draw 2:}\; lawrence yun said the rate of foreclosure filings in the united states jumped in february to a record high.\par
\vspace{2pt}\hrule\vspace{3pt}
\textbf{\scriptsize Draw 3:}\; david blitzer said he expects the rate to remain at its current level of 5. 25 percent through the end of the year.\par
\vspace{2pt}\hrule\vspace{3pt}
\textbf{\scriptsize Draw 4:}\; lawrence yun said the rate of home foreclosures has dropped to a record low.
\end{samplebox}
\end{minipage}

\vspace{0.4em}

\begin{minipage}[b]{0.48\textwidth}
\begin{samplebox}[title={\scriptsize\sffamily\bfseries $\tau=0.7$}]
\textbf{\scriptsize Draw 1:}\; lawrence yun said the rate cut was an important step in the fed ' s efforts to ease the mortgage crisis.\par
\vspace{2pt}\hrule\vspace{3pt}
\textbf{\scriptsize Draw 2:}\; lawrence yun said the bank of america ( boa ) program has helped to keep homeowners in the market.\par
\vspace{2pt}\hrule\vspace{3pt}
\textbf{\scriptsize Draw 3:}\; doug duncan said the increase in activity was driven by a weak dollar, the need for home buyers to take advantage of lower prices and a recent jump in mortgage rates.\par
\vspace{2pt}\hrule\vspace{3pt}
\textbf{\scriptsize Draw 4:}\; lawrence yun said that the 1 percent rise in january ' s home sales pushed home sales up to the highest level in two years.
\end{samplebox}
\end{minipage}\hfill
\begin{minipage}[t]{0.48\textwidth}
\begin{samplebox}[title={\scriptsize\sffamily\bfseries $\tau=1.0$}]
\textbf{\scriptsize Draw 1:}\; lawrence yun says cancellation rates remain at a 17 - year low.\par
\vspace{2pt}\hrule\vspace{3pt}
\textbf{\scriptsize Draw 2:}\; mark zandi, with whom fannie and freddie are often credited, said the nation ' s homeowners remain unsettled by tightening credit and falling home values.\par
\vspace{2pt}\hrule\vspace{3pt}
\textbf{\scriptsize Draw 3:}\; lawrence yun is predicting that 1. 8 million homeowners will fall into foreclosure by the end of october.\par
\vspace{2pt}\hrule\vspace{3pt}
\textbf{\scriptsize Draw 4:}\; jason giambi thinks that perhaps according to the full mortgage rate quotes from some economists may be too high, while banking executives are hoping for a rate hike.
\end{samplebox}
\end{minipage}

\caption{Temperature-controlled generation with K-Forcing($k{=}1$) on a held-out
LM1B prefix. Each box shows four independent draws (different random noise
$\mathbf{z}$) at a fixed temperature. }
\label{fig:temperature}
\end{figure*}

\subsection{Temperature-Controlled Generation}
\label{app:4:experiment:temperature}

We present a preliminary qualitative experiment demonstrating that K-Forcing
supports effective temperature control at inference time without retraining.

\textbf{Setup.}
For simplicity, we conduct this experiment only on LM1B with $k{=}1$. We
train a temperature-conditioned K-Forcing($k{=}1$) model from the AR teacher using
Algorithm~\ref{alg:pflm_stage1}, following the same recipe as in
Appendix~\ref{app:4:experiment:config}. The only difference is the
temperature conditioning: at each training iteration, a temperature $\tau$
is sampled uniformly from $[0.01,\,1.0]$ independently for each sample in
the mini-batch. The sampled $\tau$ is used both to scale the teacher logits
(i.e., the teacher samples from $\mathrm{softmax}(\mathbf{l}/\tau)$) and to
condition the noise encoder, as described in
Section~\ref{sec:3:push:impl}. 

\textbf{Example behavior.}
As an example, we take one held-out LM1B prefix and generate four
independent completions (with different random noise draws) at each of four
fixed temperatures $\tau\in\{0.01,\,0.3,\,0.7,\,1.0\}$ using the trained
K-Forcing($k{=}1$) checkpoint; results are shown in
Figure~\ref{fig:temperature}. At low temperatures, all four draws produce
nearly identical outputs, reflecting near-deterministic behavior. As $\tau$
increases, the draws diverge progressively, yielding more diverse and
creative---but occasionally less coherent---completions. This behavior is
qualitatively consistent with temperature scaling in standard AR sampling,
suggesting that the push-forward mapping successfully learns to modulate
output diversity in response to the temperature conditioning signal. We
emphasize that this is a non-rigorous qualitative demonstration; a systematic
study of temperature calibration (e.g., verifying that the entropy of K-Forcing
samples matches that of the AR teacher at each $\tau$) is left for future
work.

\end{document}